%% file: main.tex
\documentclass{article}

\usepackage[preprint]{neurips_2025}
\usepackage{amsmath,amssymb,algorithm,algorithmic}
\usepackage{enumitem} 
\usepackage{graphicx}
\usepackage{subcaption}    % sub-figures with their own labels
\usepackage{adjustbox}

\usepackage{amsthm} 
\theoremstyle{plain}
\newtheorem{proposition}{Proposition}

\usepackage[utf8]{inputenc} % allow utf-8 input
\usepackage[T1]{fontenc}    % use 8-bit T1 fonts
\usepackage{hyperref}       % hyperlinks
\usepackage{url}            % simple URL typesetting
\usepackage{booktabs}       % professional-quality tables
\usepackage{amsfonts}       % blackboard math symbols
\usepackage{nicefrac}       % compact symbols for 1/2, etc.
\usepackage{microtype}      % microtypography
\usepackage{xcolor}         % colors

\title{Post-Hoc Guidance for Consistency Models by Joint Flow Distribution Learning}

\author{%
  Chia-Hong Hsu \\
  Brown University\\
  \texttt{chia\_hong\_hsu@brown.edu} \\
  \And
  Randall Balestriero \\
  Brown University\\
  \texttt{randall\_balestriero@brown.edu} \\ 
}

\begin{document}

\maketitle

\input{content/abstract}
\input{content/introduction}
\input{content/background}
\input{content/method}

\input{content/experiments}

\input{content/discussion}

%%%%%%%%%%%%%%%%%%%%%%%%%%%%%%%%%%%%%%%%%%%%%%%%%%%%%%%%%%%%

%\begin{ack}

% Acknowledgement section -> suppressed at submission.

%\end{ack}

%%%%%%%%%%%%%%%%%%%%%%%%%%%%%%%%%%%%%%%%%%%%%%%%%%%%%%%%%%%%

\bibliographystyle{plain}
\bibliography{references}

%%%%%%%%%%%%%%%%%%%%%%%%%%%%%%%%%%%%%%%%%%%%%%%%%%%%%%%%%%%%
\newpage
\appendix

\input{content/appendix}

%%%%%%%%%%%%%%%%%%%%%%%%%%%%%%%%%%%%%%%%%%%%%%%%%%%%%%%%%%%%

% \input{content/neurips_checklist}
\end{document}

%% file: content/abstract.tex
\begin{abstract}
Classifier-free Guidance (CFG) lets practitioners trade-off fidelity against diversity in Diffusion Models (DMs). The practicality of CFG is however hindered by DMs sampling cost. On the other hand, Consistency Models (CMs) generate images in one or a few steps, but existing guidance methods require knowledge distillation from a separate DM teacher, limiting CFG to Consistency Distillation (CD) methods. We propose Joint Flow Distribution Learning (JFDL), a lightweight alignment method enabling guidance in a pre-trained CM. With a pre-trained CM as an ordinary differential equation (ODE) solver, we verify with normality tests that the variance-exploding noise implied by the velocity fields from unconditional and conditional distributions is Gaussian. In practice, JFDL equips CMs with the familiar adjustable guidance knob, yielding guided images with similar characteristics to CFG. Applied to an original Consistency Trained (CT) CM that could only do conditional sampling, JFDL unlocks guided generation and reduces FID on both CIFAR-10 and ImageNet 64×64 datasets. This is the first time that CMs are able to receive effective guidance post-hoc without a DM teacher, thus, bridging a key gap in current methods for CMs.
\end{abstract}

%% file: content/introduction.tex
\section{Introduction}
\label{sec:intro}
      
Diffusion models (DMs) have emerged as a powerful class of generative models, achieving remarkable success in various domains of artificial intelligence \cite{ho2020denoisingdiffusionprobabilisticmodels,
song2020generativemodelingestimatinggradients,      
song2021scorebasedgenerativemodelingstochastic,      
nichol2021improveddenoisingdiffusionprobabilistic,      
song2022denoisingdiffusionimplicitmodels,      
rombach2022high,podell2023sdxlimprovinglatentdiffusion, lin2023magic3dhighresolutiontextto3dcontent,sauer2024adversarial,karras2024analyzingimprovingtrainingdynamics}. Their ability to generate high-fidelity samples has been demonstrated in tasks such as text-to-image synthesis, speech synthesis, and video generation \cite{zhang2023addingconditionalcontroltexttoimage, ho2022video, xing2024survey, huang2022fastdiff}. These models operate through an iterative denoising process, where they gradually transform a noisy input into a structured data sample \cite{ho2020denoisingdiffusionprobabilisticmodels,luo2022understandingdiffusionmodelsunified}. The field of DMs has seen extensive research in areas like denoising schedulers, network architectures, controllability, and distillation techniques aimed at improving their performance and efficiency \cite{karras2022elucidatingdesignspacediffusionbased,karras2024analyzingimprovingtrainingdynamics,lu2022dpm,lu2022dpm_plus,salimans2022progressive,meng2023distillation,peebles2023scalable,huang2022prodiff,wimbauer2024cachecanacceleratingdiffusion}.

Classifier-Free Guidance (CFG) \cite{ho2022classifier} is a widely adopted technique that allows for controlling the generation process in DMs. CFG involves jointly training a conditional score and an unconditional score, by often using a “null’’ label for the unconditional case \cite{ho2022classifier,bradley2024classifier}. A key advantage of CFG is its simplicity and post-hoc nature, as the guidance effect is realized during inference without particularly learning a guided path for generation \cite{ho2022classifier,sadat2024no}. By interpolating between the conditional and unconditional predictions using a guidance scale, CFG enables a trade-off between the fidelity of the generated samples and their diversity \cite{karras2024guiding}. Due to its effectiveness in controllability, CFG has become a standard technique in DM applications, with guidance often preferred over unguided generation \cite{sadat2024no,yuan2023physdiff,hsu2025beyond,karras2024guiding,meng2023distillation,chung2024cfgmanifoldconstrainedclassifierfree,luo2023latentconsistencymodelssynthesizing}. 

Consistency Models (CMs) are a new family of generative models designed for fast sampling, often generating high-quality images in just one or a few steps \cite{song2023consistency,song2023improved,kim2023consistency,kim2024generalized,zheng2024trajectory,wang2024stableconsistencytuningunderstanding,lee2025truncatedconsistencymodels,geng2024consistencymodelseasy,li2025bidirectionalconsistencymodels,dai2024motionlcmrealtimecontrollablemotion,lu2025simplifyingstabilizingscalingcontinuoustime,yang2024consistencyflowmatchingdefining,he2024consistencydiffusionbridgemodels,wang2024stableconsistencytuningunderstanding}. They achieve this speed by learning to directly map noisy inputs to clean data samples through two main approaches \cite{song2023consistency}: Consistency Distillation (CD) \cite{li2025bidirectionalconsistencymodels,luo2023latentconsistencymodelssynthesizing,zheng2024trajectory,li2024connecting,fei2024musicconsistencymodels}, which involves distilling knowledge from a pre-trained DM, and Consistency Training (CT), where the training procedure is totally data-driven \cite{song2023improved,li2025bidirectionalconsistencymodels,geng2024consistencymodelseasy,dao2025improvedtrainingtechniquelatent,hsu2025beyond}. Existing guidance methods for CMs typically rely on knowledge distillation of CFG from a teacher DM, which inherently binds guidance to the CD approach for CMs \cite{luo2023latentconsistencymodelssynthesizing,chen2024pixartdeltafastcontrollableimage}. In DMs, the benefits of CFG can be directly evaluated at inference time using the same model \cite{karras2024guiding}. For distilled CMs, however, guidance effects are inherently tied to training, making the learning task more complex than simply fitting an unguided ordinary differential equation (ODE) trajectory \cite{luo2023latentconsistencymodelssynthesizing, zheng2024trajectory}. This makes it challenging to directly evaluate the benefits of guidance since it would require training a separate unguided version for a fair comparison. To conclude, current guidance approaches for CMs are largely dependent on the existence of a DM teacher, and the isolated impact of guidance on CMs remains unclear without a direct comparison to another CM that only does unguided generation.

% Paragraph 4, the paragraph that talks about our work. 
In this paper, we introduce a novel post-hoc guidance method for CMs that operates independently of DMs. Starting with a pre-trained unguided CM as an ODE solver of the diffusion path, our method enables guidance learning by interpolating the directions of synthesized unconditional and conditional distributions. We call this \textbf{Joint Flow Distribution Learning (JFDL)} and provide insights of why it works based on the connection to Flow-based generative models (FMs) \cite{Kobyzev_2021,lipman2023flowmatchinggenerativemodeling,yang2024consistencyflowmatchingdefining,gao2025diffusionmeetsflow,liu2022flowstraightfastlearning,liu2022rectifiedflowmarginalpreserving,frans2024stepdiffusionshortcutmodels,kim2024simplereflowimprovedtechniques}. Furthermore, we discovered that our algorithm can be adapted to work effectively without an unconditional ODE solver. This broadens the generality of our method for pre-trained CMs as training them can be a complex endeavor. Our algorithm offers a post-hoc way to equip an unguided CM with guidance capabilities, effectively bridging the gap for CT models. We access the effectiveness of guidance by showing significant FID improvements on both CIFAR-10 \cite{cifar10} and ImageNet 64x64 datasets \cite{5206848} compared to the initial unguided CM. In summary, our contributions are:
\begin{itemize}
\item We propose JFDL, a novel post-hoc guidance method for CMs that does not rely on DMs.
\item We provide insights to JFDL's effectiveness from a FM's perspective.
\item We demonstrate that a pre-trained CM without explicit design for unconditional sampling, remain effective for JFDL and guidance tuning for CMs. 
\item We demonstrate FID improvements on CIFAR-10 and ImageNet 64x64 by applying our method to CT models.
\end{itemize} 

%% file: content/background.tex
\section{Preliminaries}

This section introduces the prelimiaries for the connection between Diffusion models (DMs) and Flow-based models (FMs), Consistency Models (CMs), and Classifier-free Guidance (CFG). We also reuse the established notations for the rest of the paper. 

\paragraph{Diffusion Models and Flow-based Models, Two Sides of the Same Coin.} DMs are probabilistic generative models that define a forward process which gradually adds noise to a data sample $\mathbf{x}_0 \sim p_{\text{data}}(\mathbf{x})$ over time $t \in [0,1] $. For the rest of our paper, we focus on the Variance Exploding (VE) scheme \cite{song2021scorebasedgenerativemodelingstochastic,karras2022elucidatingdesignspacediffusionbased} that has the stochastic differential equation (SDE), \(\text{d} \mathbf{x}_t = \sqrt{2 \sigma_t }\text{d}\mathbf{w}_t\), where $\mathbf{w}_t$ is a standard Wiener process, and $\sigma_t=\sigma_\text{max}t$ is the noise exploding term with $ \sigma_\text{max} \gg 1 $. The reverse process--generating data from noise--can also be described by an SDE \cite{song2021scorebasedgenerativemodelingstochastic}. Notably, there exists a Probability Flow ODE (PF-ODE), that shares the same marginal probability densities as the reverse process, $ \text{d} \mathbf{x}_t / \text{d}\sigma_t = -\sigma_t \nabla_{\mathbf{x}_t}\log p_t(\mathbf{x}_t) = 1 / \sigma_t \cdot ( \mathbf{x}_t - \mathbb{E}_{\mathbf{x}_0} \left[ \mathbf{x}_0 | \mathbf{x}_t \right] ) $, where $\nabla_{\mathbf{x}_t}\log p_t(\mathbf{x}_t)$ is called the score function \cite{kim2023consistency,song2023consistency}. A common loss function for predicting the expected value of $ \mathbf{x}_0 $ has the form:
\begin{equation}
\mathcal{L}_{DM}(\theta) = \mathbb{E}_{t, \mathbf{x}_0, \mathbf{x}_t} \left[ w(t) \| D_\theta(\mathbf{x}_t, t) - \mathbf{x}_0 \|_2^2 \right], 
\end{equation}
where $w(t)$ is a weighting function dependent on time, and $\theta$ parametrizes the model $D_\theta$ \cite{karras2022elucidatingdesignspacediffusionbased}. Other training objectives, such as predicting score, noise, or even a mixture of noise and $ \mathbf{x}_0 $, can be formulated similarly by reparameterizating the loss term \cite{luo2022understandingdiffusionmodelsunified,gao2025diffusionmeetsflow,lu2025simplifyingstabilizingscalingcontinuoustime}. 

FMs learn a map between an untractable distribution $p_0(\mathbf{x}) $, e.g., $p_\text{data}(\mathbf{x})$, and a simple distribution $p_1(\mathbf{x})$ , e.g., the standard normal \cite{Kobyzev_2021,lipman2023flowmatchinggenerativemodeling,liu2022rectifiedflowmarginalpreserving}. We define an ODE with the time dependent vector field $u_t$, and the flow $\phi_t$ by:
$\text{d} \phi_t(\mathbf{x}) / \text{d}t  = u_t (\phi_t(\mathbf{x})), \ \phi_0(\mathbf{x})=\mathbf{x}_0 $. Since we do not have access to $u_t$ that satisfies the marginal densities $p_t$, a per sample aggregation of \textit{conditional vector fields} (CVF) $ u_t(\mathbf{x}_t | \mathbf{x}_0) $ can be used to construct the \textit{conditional flow matching} (CFM) objective \cite{lipman2023flowmatchinggenerativemodeling} as follows: 
\begin{equation}
\mathcal{L}_{CFM}(\theta) = \mathbb{E}_{t, \mathbf{x}_0, \mathbf{x}_t } \left[ \| F_\theta(\mathbf{x}_t, t) - u_t(\mathbf{x}_t | \mathbf{x}_0) \|_2^2 \right] , 
\end{equation}
where $\mathbf{x}_t \sim p_t(\mathbf{x} | \mathbf{x}_0)$, and $\theta$ parametrizes the model $F_\theta$. If we consider the VE scheme from diffusion, we can construct a \textit{probability path} $p_t(\mathbf{x} | \mathbf{x}_0) = \mathcal{N}(\mathbf{x}; \mu_t(\mathbf{x}_0), \sigma_t^2\mathbf{I})$, where $\mu_t(\mathbf{x}) = \mathbf{x}, \sigma_t=\sigma_\text{max}t$, and the CVF has the analytical form \cite{lipman2023flowmatchinggenerativemodeling}: 
\begin{equation}
\label{eq:cvf}
u_t(\mathbf{x}_t | \mathbf{x}_0) = \frac{\sigma_t'}{\sigma_t} (\mathbf{x}_t - \mu_t(\mathbf{x}_0)) + \mu_t'(\mathbf{x}_0). 
\end{equation}
With adjustments in parametrization and weighting, previous works have shown that the training target of DMs and FMs are translatable \cite{gao2025diffusionmeetsflow,lu2025simplifyingstabilizingscalingcontinuoustime}. We will explicitly show the connection of $u_t(\mathbf{x}_t | \mathbf{x}_0)$ and CMs in Section \ref{sec:method-overview}.

\paragraph{Consistency Models.} CMs are a class of generative models designed for one or few-step sampling by learning to directly map any point on the PF-ODE trajectory to its endpoint (the data sample) \cite{song2023consistency,kim2023consistency}. In our study, we focus on a family of continuous-time CT models introduced in ECT \cite{geng2024consistencymodelseasy} due to its significantly reduced GPU resources at training. Based on the VE diffusion scheme, the ECT objective is as follows:
\begin{equation}
\mathcal{L}_{ECT}(\theta) = \mathbb{E}_{t, r, \mathbf{x}_0, \mathbf{z}} \left[ \ d(G_\theta(\mathbf{x}_0+\sigma_t\mathbf{z}, t), \  G_{\theta^-}(\mathbf{x}_0+\sigma_r\mathbf{z}, r))\  \right],
\end{equation}
where $\mathbf{x}_0 \sim p_\text{data}$ is a data sample, $\mathbf{z} \sim \mathcal N(\mathbf{0}, \mathbf{I})$ is a sampled noise, $t,r$ are consecutive time steps, and $d(\cdot,\cdot)$ is a distance function. $G_\theta$ is the online CM, where $\theta^{-}$ denotes the stop gradient of the same CM being the target. The distance $\Delta:=|\sigma_t-\sigma_r|$ is approximately $\sigma_t/q^\text{n}$, where n is an integer denoting the training stage. With $q=2$ and $q=4$ at the start of CIFAR-10 and ImageNet 64x64 experiments respectively, $\Delta$ progressively decreases by a factor of $q$ at training, and the factor at the final stage for both is $q^{n} = 256$.

\paragraph{Classifier-Free Guidance in DMs and CMs.} CFG is a technique used to controlling the guidance of the generation process \cite{ho2022classifier,luo2023latentconsistencymodelssynthesizing}. In DMs, this is done nearly training-free by extrapolating the conditional and unconditional updates at sampling $D^{\text{CFG}}_{\theta}(\mathbf{x}_t, t, c, \omega) = \omega \cdot \frac{[\mathbf{x}_t-D_\theta(\mathbf{x}_t,t,c)]}{\sigma_t} + (1-\omega) \cdot \frac{[\mathbf{x}_t - D_\theta(\mathbf{x}_t,t,\emptyset)]}{\sigma_t} $, with $\omega \ge 1$, 
and $c,\emptyset$ are the conditional and unconditional classes respectively \cite{ho2022classifier}. To achieve CFG in CMs, previous works rely on knowledge distillation of a CFG teacher DM, $D^\text{CFG}_\theta$, at training, limiting CFG to CD models \cite{zheng2024trajectory,luo2023latentconsistencymodelssynthesizing}. We will demonstrate a DM-free CFG method for CMs which uses the pre-trained CM itself as an ODE solver. As a result, we can access the benefits of CFG post-hoc compared to the original CM and equip CT models with guidance. 

%% file: content/method.tex
\section{Guidance via Joint-Flow Distribution Learning}
\label{sec:method}
We begin by discussing the connection of FMs with CMs. Following, we introduce the naive JFDL for post-hoc guidance tuning, which has a prerequisite of a pre-trained class-conditioned CM that learns the unconditional PF-ODE (with the $\emptyset$ class). We will provide a theoretical analysis of the pseudo-noise in JFDL, supported by experimental verification. In the end, we find out that an adjusted JFDL algorithm works surprisingly well without the need of the $\emptyset$ solver from CM. 

\begin{algorithm}[t]
  \caption{Naive JFDL for Post-Hoc Guidance}
  \label{alg:naive-jfdl}
  \begin{algorithmic}[1]
  
    \STATE \textbf{Input:} dataset $\mathcal{D}$, pre-trained CM $\psi$, weighting function $w(t)$, timesteps sampling density $p(t,r)$, total iterations $\mathrm{totalIters}$, max guidance scale $\omega_\text{max}$, gradnorm layer $\theta_\text{gn}$, gradnorm function $f(L_\text{ECT}, L_\text{JFDL} \ | \ \theta_\text{gn})$.
    
    \STATE \textbf{Init:} $\theta \leftarrow \psi$, $\mathrm{Iters} \leftarrow 0$.
    
    \WHILE{$\mathrm{Iters} < \mathrm{totalIters}$}
    
      \STATE Sample $(\mathbf{x}_0^c, c) \sim \mathcal{D}$, $t,r \sim p(t,r)$, $\mathbf{z} \sim \mathcal{N}(\mathbf{0},\mathbf{I})$, $\omega=1$;
      
      \STATE $\mathbf{x}_t^c \leftarrow \mathbf{x}_0^c + \sigma_t \mathbf{z}$; $\mathbf{x}_r^c \leftarrow \mathbf{x}_0^c + \sigma_r \mathbf{z}$; 
             
      \STATE $L_\text{ECT}(\theta) \leftarrow 
      w(t) \ d \bigl (G_{\theta}(\mathbf{x}_t^c, t, c, \omega), \ G_{\theta^-}(\mathbf{x}_r^c, r, c, \omega) \bigr)$; \hfill $\triangleright$ ECT loss

      \STATE Sample $\mathbf{y}_0^c \leftarrow G_{\psi^-}(\mathbf{x}_t^c, t, c)$, $\mathbf{y}_0^{\emptyset,t} \leftarrow G_{\psi^-}(\mathbf{x}_t^c, t, \emptyset)$, $\mathbf{z'} \sim \mathcal{N}(\mathbf{0},\mathbf{I})$, $\omega \sim \mathcal{U}(1,\omega_\text{max})$;

      \STATE $\mathbf{y}_t^c \leftarrow \mathbf{y}_0^c + \sigma_t \mathbf{z'}$; $\mathbf{y}_r \leftarrow \mathbf{y}_t^c + \{ \omega \left[ \frac{\mathbf{y}_t^c-\mathbf{y}_0^c}{\sigma_t} \right] + (1-\omega) \left[ \frac{\mathbf{y}_t^c-\mathbf{y}_0^{\emptyset,t}}{\sigma_t} \right] \} \cdot (\sigma_r-\sigma_t) $; 

      \STATE $L_\text{JFDL}(\theta) \leftarrow 
      w(t) \ d \bigl (G_{\theta}(\mathbf{y}_t^c, t, c, \omega), \ G_{\theta^-}(\mathbf{y}_r, r, c, \omega) \bigr)$; \hfill $\triangleright$ JFDL loss

      \STATE $\lambda_\text{gn} = f(L_\text{ECT}, L_\text{JFDL} \ | \ \theta_\text{gn})$;
      \STATE $L(\theta) = L_\text{ECT}(\theta) + \lambda_\text{gn} L_\text{JFDL}(\theta)$;
              
      \STATE $\theta \gets \theta - \eta\nabla_\theta L(\theta)$;
      \STATE $\mathrm{Iters} \leftarrow \mathrm{Iters}+1$;
    \ENDWHILE
    \STATE \textbf{return} $\theta$ 
  \end{algorithmic}
\end{algorithm}

% laplace distribution outliers, papers??
\subsection{Routine Correspondence between FMs and CMs}
\label{sec:method-overview}
Following the VE scheme, let $p_0(\mathbf{x})=p_\text{data}(\mathbf{x})$ be the data distribution, where $\mathbf{x} \in \mathbb{R}^d$, and let $p_0(\mathbf{x}|c)=p_\text{data}(\mathbf{x}|c)$ be the class-conditioned data distribution, where $c \in \mathcal{C}$, s.t. $p_0(\mathbf{x})=\int p_0(\mathbf{x}|c)p(c)\ \text{d}c$. 

To train a FM, we begin by sampling $\mathbf{x}_0^c \sim p_\text{data}(\mathbf{x}|c)$ and $\mathbf{z} \sim \mathcal{N}(\mathbf{0},\mathbf{I})$. We can construct a class probability path $p_t(\mathbf{x}|c)$ with $p_t(\mathbf{x}|\mathbf{x}_0^c,c) = \mathcal{N}(\mathbf{x};\mathbf{x}_0^c,\sigma_t^2\mathbf{I})$, satisfying $p_t(\mathbf{x}|c) = \int p_t(\mathbf{x}|\mathbf{x}_0^c,c)p_0(\mathbf{x}_0^c|c) \ \text{d}\mathbf{x}_0^c$. Observe that $p_1(\mathbf{x}|\mathbf{x}_0^c, c) \approx \mathcal{N}(\mathbf{x};\mathbf{0},\sigma_\text{max}^2\mathbf{I}) $ is Gaussian, and $p_0(\mathbf{x}|\mathbf{x}_0^c,c)$ is the class-conditioned data distribution \textit{conditioned on $\mathbf{x}_0^c$}, we can analytically derive the CVF, $u_t(\mathbf{x}_t|\mathbf{x}_0^c,c)$, as the target of the CFM objective. To train a CM, we choose a $t$, and noisify the sampled data, $\mathbf{x}^c_t = \mathbf{x}^c_0 + \sigma_tz$. Then, we can compute the denoising direction as $\mathbf{u} := 1/t \cdot (\mathbf{x}^c_t - \mathbf{x}^c_0)$, and derive the target for the consistency loss as $G_{\theta^-}(\mathbf{x}_t+\mathbf{u} \cdot(\sigma_r-\sigma_t),r,c)$. 

We highlight the correspondence of the routines between FMs and CMs. Under the VE scheme, both CM and FM learns a mapping between the data distribution $p_0$ and Gaussian $p_1$. The probability path in FM is equivalent to the noisifying step in CM. Also, we can derive that the CVF is $\sigma_\text{max}\cdot\mathbf{z}$ following \eqref{eq:cvf}, which equates to the denoising direction in CM. Hence, even without a DM teacher, the ECT objective recovers an accurate approximation of the PF-ODE vector field, just as the CFM objective lets FM estimate the data-marginal velocity field. As we introduce JFDL, we will provide insights from these correspondences into how it effectively captures the mapping associated with the unconditional class distribution.

\paragraph{Joint-Flow Distribution Learning.} In the variance–exploding (VE) setting, we assume the existence of a \textit{perfect ODE solver},
\begin{equation}
  \mathrm{Solver}\colon \ 
  \mathbb{R}^{d}\times[0,1]^{2}\times\mathcal{C}
  \longrightarrow
  \mathbb{R}^{d}, 
  \quad
  (\mathbf{x},t,s,c)
  \longmapsto
  \mathrm{Solver}_{t\to s}^{c}(\mathbf{x}),
\end{equation}
where the shorthand $\mathrm{Solver}_{t\to s}^{c}(\mathbf{x})
:=\mathrm{Solver}(\mathbf{x},t,s,c)$ will be used throughout. For class $c \in \mathcal{C},$ time indices $t,s\in[0,1]$, $\mathrm{Solver}$ satisfies the push-forward condition
\begin{equation}
\bigl(\mathrm{Solver}_{t\!\to s}^{\,c}\bigr)_{\#}\,
  p_{t}(\mathbf{x}\mid c)
  \;=\;
  p_{s}(\mathbf{x}\mid c),
\end{equation}

so that it deterministically transports the class-conditioned distribution from noise level \(t\) to level \(s\). Similar to CFG, we introduce a special “null’’ label
$\emptyset\in\mathcal{C}$.  
Choosing $c=\emptyset$ collapses the conditional path to the
\emph{unconditional} path: $p_t(\mathbf{x}|\emptyset):=p_t(\mathbf{x})
  \ \Longrightarrow \ (\mathrm{Solver}_{t\to s}^{\emptyset})_{\#}p_t(\mathbf{x})=p_s(\mathbf{x})$,
Thus, the same solver acts as a one-step denoiser (when $s=0$) for \textit{both} the class-conditioned and
$\emptyset$ trajectories. With access to an ideal $\mathrm{Solver}$, we introduce \textit{Joint-Flow Distribution Learning} (JFDL) that constructs a pair of denoising directions, $\mathbf{u}^{\text{cls}}$ and $\mathbf{u}^{\emptyset}$, to embody CFG for the consistency loss. The routine is:

\begin{enumerate}[leftmargin=*,label=(\roman*)]
\item \textbf{Draw sample.} Draw a class image $\mathbf{x}_0^{c}\sim p_{0}(\mathbf{x}|c)$ and choose a source noise index $t$.
\item \textbf{Compute unconditional anchor.} Define
        $\mathbf{y}_{0}^{\emptyset,t}
        :=\mathrm{Solver}_{t\to 0}^{\emptyset}\bigl(
          \mathrm{Solver}_{0\to t}^{c}(\mathbf{x}_0^{c})\bigr)$.
\item \textbf{Noisify.} Add noise by $\mathbf{x}_{t}^{c}=\mathbf{x}_0^{c}+\sigma_{t}\mathbf{z}$ with $\mathbf{z}\sim\mathcal{N}(\mathbf{0},\mathbf{I})$. 
\item \textbf{Compute denoising vectors.} Define $\mathbf{u}^{\text{cls}} :=\dfrac{\mathbf{x}_{t}^{c}-\mathbf{x}_0^{c}}{\sigma_{t}}$ and $\mathbf{u}^{\emptyset} :=\dfrac{\mathbf{x}_{t}^{c}-\mathbf{y}_{0}^{\emptyset,t}}{\sigma_{t}}$.
\item \textbf{Extrapolate guidance.} Extrapolate with $\omega$, define $\mathbf{u}^{\omega}:=\omega\,\mathbf{u}^{\text{cls}}+(1-\omega)\,\mathbf{u}^{\emptyset}$.
\item \textbf{Project to lighter noise for consistency loss.} Pick $r<t$ and let
        $\mathbf{x}_{r}
        =\mathbf{x}_{t}^{c}+(\sigma_{r}-\sigma_{t})\,\mathbf{u}^{\omega}$. Compute loss as $d\bigl(
          G_{\theta}(\mathbf{x}_{t}^{c},t,c,\omega),\;
          G_{\theta^{-}}(\mathbf{x}_{r},r,c,\omega)
        \bigr)$.
\end{enumerate}

% font size larger, keep left y-axis, keep bottom x-axis.
\begin{figure}[t]
  \centering
  \includegraphics[width=\linewidth]{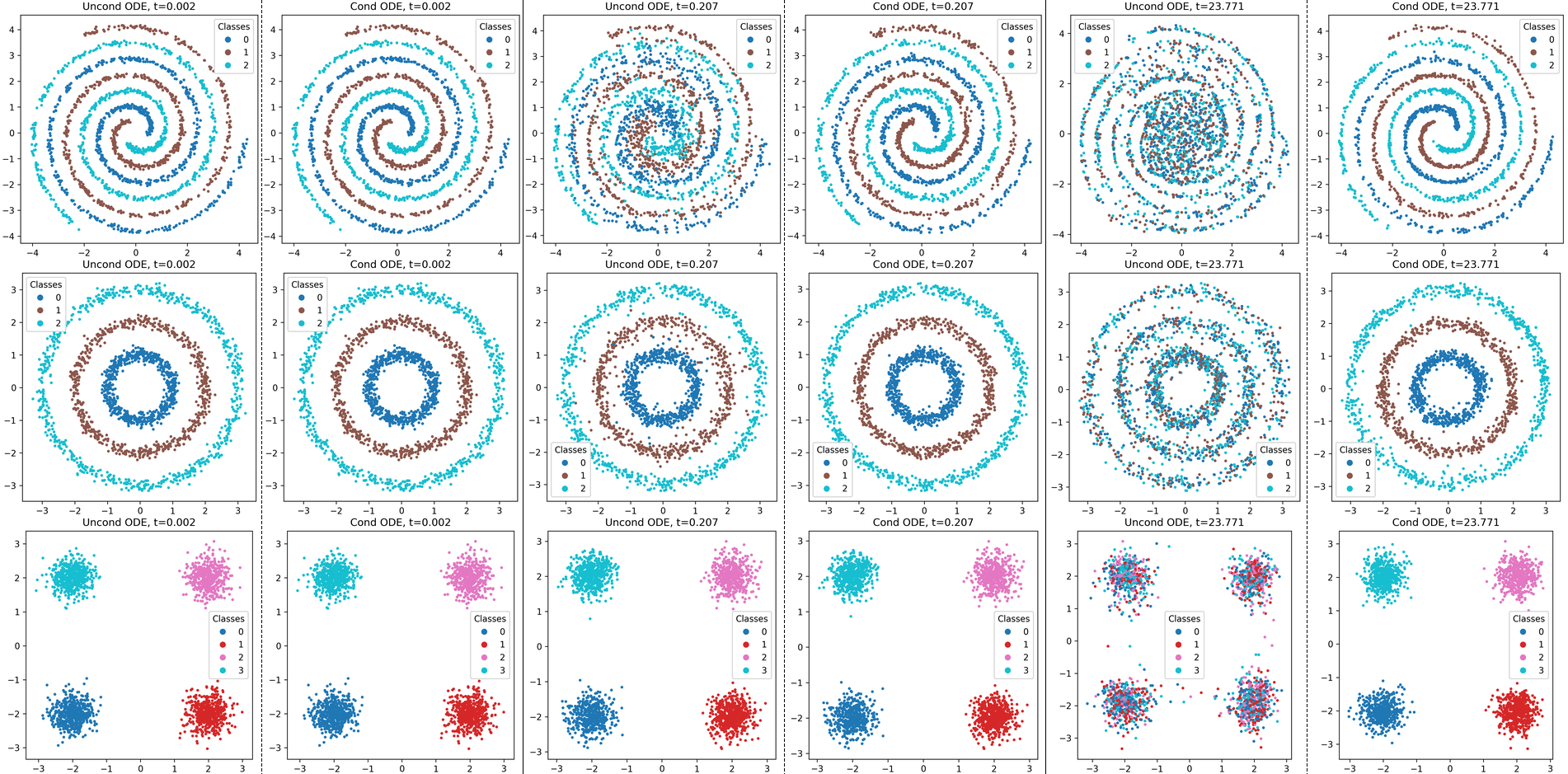}
  \caption{
    \textbf{Comparison of class v.s. $\emptyset$-conditioned ODE solutions and their marginal distributions.}
    Each row corresponds to a different 2-D toy dataset, \textit{spiral}, \textit{circle}, \textit{Gaussian blob}. We choose three different time steps, $\sigma_t=0.002, 0.207, 23.771$ (from left to right), and compare the distribution of a hybrid flow $p(\mathbf{y}_{0}^{\emptyset,t})$ (\textbf{left}) with the class distribution $p(\mathbf{x}_0^c|c)$ (\textbf{right}) separated by dotted lines. In Prop. \ref{prop:hybrid-distribution}, we state that the marginal of the two are the same.
  }
  \label{fig:toy-dist}
\end{figure}

% See figure on page 5
When $\omega=1$, the extrapolated direction $\mathbf{u}^\omega = \mathbf{u}^\text{cls}$ is no different than the denoising direction for unguided CMs. One one hand, for small noise index $t$, the \textit{hybrid} ODE solution $\mathbf{y}^{\emptyset,t}_0 \approx \mathbf{x}_0^c$, while for large $t$, the difference is increased, suggesting the distribution deviates more. On the other hand, for any chosen $t$, the distribution of the hybrid ODE solution integrated over $\mathcal{C}$ is always the data distribution $p_0(\mathbf{x})$. See illustration of Fig.\ref{fig:toy-dist}. Formally, we state that:

\begin{proposition}[Hybrid flow preserves the marginal data distribution]
\label{prop:hybrid-distribution}
Fix any noise index \(t\in[0,1]\).  
Draw a class label \(c\sim p(c)\) and a clean sample
\(\mathbf{x}_{0}^{c}\sim p_{0}(\mathbf{x}\mid c)\).
Define the hybrid ODE solution $\mathbf{y}_{0}^{\emptyset,t}:=\mathrm{Solver}_{t\to 0}^{\emptyset}\Bigl(\mathrm{Solver}_{0\to t}^{c}(\mathbf{x}_{0}^{c})\Bigr).$ We denote the density of $\mathbf{y}_{0}^{\emptyset,t}$ as $p(\mathbf{y}|c)$, which we define as:
\begin{equation}
p(\mathbf{y}\mid c):=\bigl(\mathrm{Solver}_{t\to0}^{\emptyset}\bigr)_{\#}\bigl(\mathrm{Solver}_{0\to t}^{c}\bigr)_{\#}p_{0}(\mathbf{x}\mid c).
\end{equation}

Then the marginal law of \(\mathbf{y}_{0}^{\emptyset,t}\) over the class prior \(p(c)\) coincides with the unconditional-class data distribution:
\begin{equation}
    \int_{\mathcal{C}} p(c)\,p(\mathbf{y}\mid c)\,\text{d}c=p_0(\mathbf{x}).
\end{equation}
\end{proposition}

We provide the proof of Prop. \ref{prop:hybrid-distribution} in Appendix \ref{sec:appendix-proofs}. To match JFDL with the FM routine, the denoising direction $\mathbf{u}^\emptyset$ points to the data distribution $p_0(\mathbf{x})$, for any $t \in [0,1]$. We can further establish its probability path $p_\tau(\mathbf{x}|\mathbf{y}_{0}^{\emptyset,t})$ to be,
\begin{equation}
\begin{aligned}
\label{eq:prob-path}
\mathcal{N}(\mathbf{x}; \mu_\tau(&\mathbf{y}_{0}^{\emptyset,t}),\sigma_\tau^2\mathbf{I}) , \text{ where }\ \mu_\tau(\mathbf{y}_{0}^{\emptyset,t})=\mathbf{y}_{0}^{\emptyset,t}+\frac{\mathbf{x}_0^c - \mathbf{y}_{0}^{\emptyset,t}}{t}  \cdot \tau , \ \sigma_\tau = \sigma_\text{max}\cdot \tau, 
\end{aligned}
\end{equation}
such that when $\tau=t$, it coincides with $\mathcal{N}(\mathbf{x};\mathbf{x}_0^c,\sigma_t^2\mathbf{I})$, i.e., \textit{joining} the denoising direction from $p_0(\mathbf{x}|\mathbf{x}_0^c,c)$ at $\mathbf{x}_t^c$. So far, we are left with matching the $p_1$ distribution. If we can show that $p_1(\mathbf{x}|\mathbf{y}_{0}^{\emptyset,t}) \approx \mathcal{N}(\mathbf{0},\sigma_\text{max}^2\mathbf{I})$ following the probability path $p_\tau(\mathbf{x}|\mathbf{y}_{0}^{\emptyset,t})$ in \eqref{eq:prob-path}, then with Prop. \ref{prop:hybrid-distribution}, we can imply that the denoising direction $\mathbf{u}^\emptyset$ in JFDL is the CVF from Gaussian to the $\emptyset$-class data distribution. Consequently, JFDL achieves CFG like guidance. 

\subsection{Pseudo-Noise Analysis and Experimental Verification}

In this section, our goal is to show that $p_1(\mathbf{x}|\mathbf{y}_{0}^{\emptyset,t})$ approximates $\mathcal{N}(\mathbf{0},\sigma_\text{max}^2\mathbf{I})$. In other words, for any $t$, we sample $\mathbf{x}_0^c \sim p_\text{data}(\mathbf{x}|c)$ and compute $\mathbf{y}_{0}^{\emptyset,t}$ to be the solution of its hybrid ODE following the JFDL routine. Evaluating the probability path in \eqref{eq:prob-path} at $\tau=1$ gives the random variable,
\begin{equation}\label{eq:pseudo-noise}
\overbrace{%
  \underbrace{%
    \mathbf{y}_{0}^{\emptyset,t}
    \;+\;
    \frac{\mathbf{x}_{0}^{c}-\mathbf{y}_{0}^{\emptyset,t}}{t}%
  }_{\text{mixed-signals}}
  \;+\;
  \sigma_{\max}\,\mathbf{z}%
}^{\text{pseudo-noise}},
\qquad
\mathbf{z}\sim\mathcal{N}(\mathbf{0},\mathbf{I}).
\end{equation}
which we will show how $\sigma_\text{max}\mathbf{z}$ dominates analytically. Then, we will verify the Gaussianity of \eqref{eq:pseudo-noise} experimentally on 2D toy datasets and CIFAR-10 with the signal-to-noise ratio and Gaussianity tests.

% p-values extra in the appendix.
\begin{figure}[t]
\centering
\includegraphics[width=\linewidth]{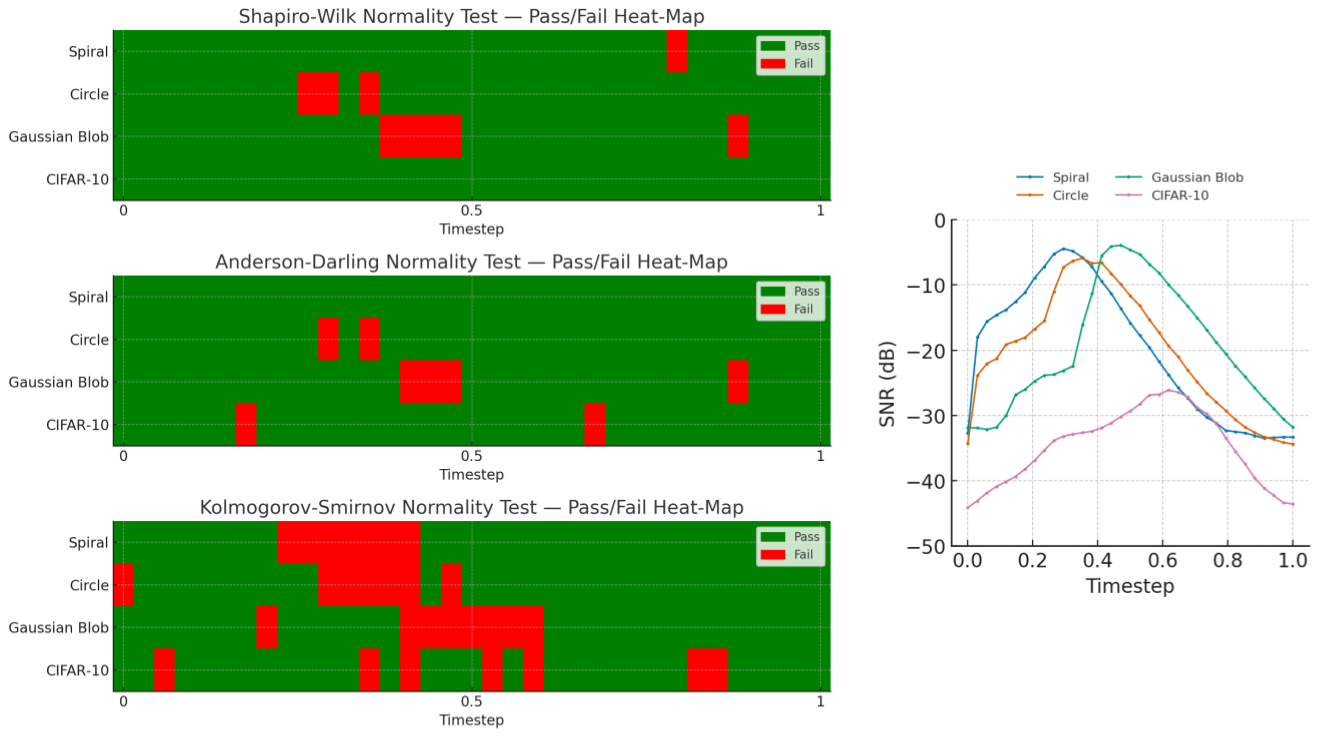}
\caption{\textbf{Normality of pseudo-noise across timesteps.} 
Heat-maps show pass (green) or fail (red) at $\alpha=0.05$ for \textbf{(top)} Shapiro–Wilk, \textbf{(middle)} Anderson–Darling, and \textbf{(bottom)} Kolmogorov–Smirnov tests. Rows correspond to the four datasets, \textit{spiral, circle, Gaussian blob, CIFAR-10}. With only a handful of isolated rejections, as well as extremely low SNR ratio, the pseudo-noise is effectively Gaussian at almost every $t$, supporting the normality assumption.}
\label{fig:normality}
\end{figure}

\paragraph{Theoretical analysis.}
Without loss of generality, the clean variables $\mathbf{y}_{0}^{\emptyset,t}$ and $\mathbf{x}_{0}^{c}$ are centered at zero and confined to the hyper-cube $[-1,1]^{d}$. With a VE schedule, we can make $\sigma_{\max}$ arbitrarily large, so for any fixed $t\gg0$ the Gaussian term $\sigma_{\max}\mathbf{z}$ in \eqref{eq:pseudo-noise} already dominates, forcing the conditional law at $\tau=1$ to be close to $\mathcal{N}(\mathbf{0},\sigma_{\max}^{2}\mathbf{I})$.

The non-trivial regime emerges in the limit $t \to 0$, where the deterministic drift term $\mu_1(t) := \mathbf{y}_{0}^{\emptyset,t} + [\mathbf{x}_{0}^{c} - \mathbf{y}_{0}^{\emptyset,t}]/t$ becomes comparable in magnitude. To analyze its contribution, we introduce a function $f(t) := \mathbf{y}_{0}^{\emptyset,t}$ representing the time-dependent evolution of the hybrid flow, and express the mixed-signals as follows:

\begin{proposition}[Mixed-signals as a function of $t$]\label{prop:pseudo-noise}
Let $f(t) := \mathbf{y}_{0}^{\emptyset,t} = \mathrm{Solver}_{t\to 0}^{\emptyset}\!\bigl(\mathrm{Solver}_{0\to t}^{c}(\mathbf{x}_0^{c})\bigr)$ denote the hybrid flow, we can define the mixed-signals $g(t)$ as:
\begin{equation}
    g(t) := f(0) + \frac{f(0) - f(t)}{t}.
\end{equation}
then, from Taylor expansion to the 2nd order, we have:
\begin{equation}\label{eq:taylor}
    g(t) \approx f(0)+\frac{t}{2}f''(0) = \mathbf{x}_0^c + \frac{\sigma_\text{max}t}{2}\nabla_{\mathbf{x}_0^c}\log p(c|\mathbf{x}_0^c). 
\end{equation}
\end{proposition}

The resulting Taylor expansion of the mixed-signal is surprisingly interpretable near $t=0^+$. As $\mathbf{x}^c_0$ genuinely belongs to class $c$, we expect the $p(c|\mathbf{x}_0^c) \approx 1$. The gradient represents how quickly the probability would change if the data were perturbed slightly, which would be small if  $p(c|\mathbf{x}_0^c)$ is close to the local maximum. We provide the proof of Prop. \ref{prop:pseudo-noise} in Appendix \ref{sec:appendix-proofs}. Our theoretical support concludes that for any $t$, $p_1(\mathbf{x}|\mathbf{y}_{0}^{\emptyset,t})$ approximates a Gaussian.

\paragraph{Experiment Verification.}\label{sec:exp-verification} In this section, we verify experimentally that for any $t$, the mixed-signal in \eqref{eq:pseudo-noise} will be dominated by the noise term $\sigma_\text{max}\mathbf{z}$. To evaluate this claim, we trained a DM and CM as an ODE solver on the 2D toy datasets and CIFAR-10 respectively. Notably, a typical DM/CM is used to solving the ODE backward from $t \rightarrow 0$, not forward. To construct a class $c$ sample and its unconditional anchor pair, we adjust the beginning few steps of JFDL as follows:
\begin{enumerate}[leftmargin=*,label=(\roman*)]
\item \textbf{Draw and noisify.} Draw $\mathbf{x}_0^{c}\sim p_{0}(\mathbf{x}|c)$, $\mathbf{z} \sim \mathcal{N}(\mathbf{0},\mathbf{I})$, choose a source noise index $t$, then noisify $\mathbf{x}_t^{c} = \mathbf{x}_0^{c} +\sigma_t\mathbf{z}$.
\item \textbf{Compute class anchor.} Define
        $\mathbf{y}_{0}^{c}
        :=\mathrm{Solver}_{t\to 0}^{c}\bigl(
          \mathbf{x}_t^{c}\bigr)$.
\item \textbf{Compute unconditional anchor.} Define
        $\mathbf{y}_{0}^{\emptyset,t}
        :=\mathrm{Solver}_{t\to 0}^{\emptyset}\bigl(
          \mathbf{x}_t^{c}\bigr)$.
\end{enumerate}
With a pre-trained DM/CM as an ODE solver, the routine aligns with the naive JFDL in Alg. \ref{alg:naive-jfdl}. Moreover, $\mathbf{y}_0^c$ replaces $\mathbf{x}_0^c$ in the pseudo-noise term. To assess the normality of the pseudo-noise for varying $t$, we employed standard statistical tests \cite{das2016brief,yap2011comparisons}, including the Shapiro-Wilk \cite{SHAPIRO1965}, Anderson-Darling \cite{anderson1952asymptotic}, and Kolmogorov-Smirnov tests \cite{16e7f618-c06b-3d10-8705-1086b218d827}, each conducted at a significance level of $\alpha=0.05$. We also show the log signal-to-noise ratio (SNR), $10\times\log_{10}\frac{\|\text{mixed-signals}\|^2}{\|\text{pseudo-noise}-\text{mixed-signals}\|^2}$. Our results (Fig. \ref{fig:normality}) indicate that we can successfully show Gaussianity of the pseudo-noise constructed by the probability path in \eqref{eq:prob-path}. We will provide the training and testing details of the toy datasets in Appendix \ref{sec:appendix-exp}.

\begin{figure}[t]
\centering
\includegraphics[width=\linewidth]{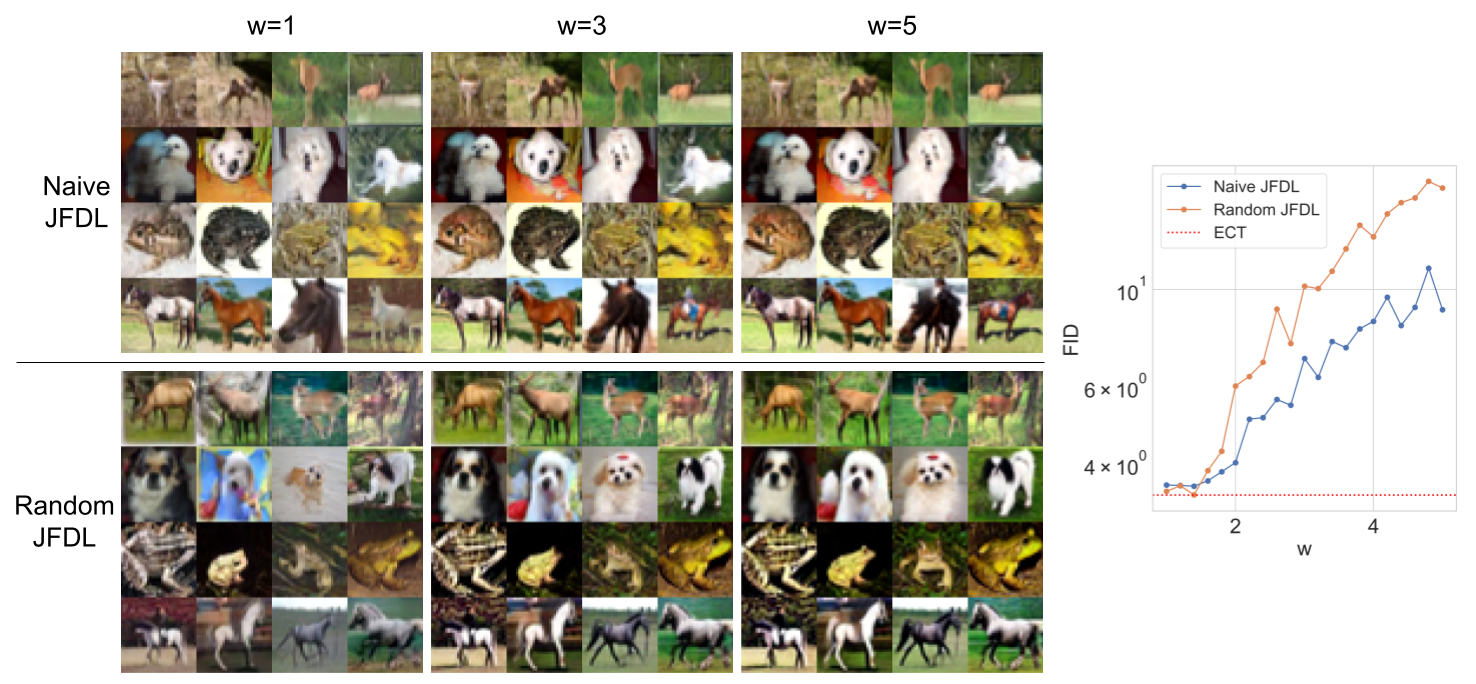}
\caption{\textbf{Preliminary results tuning $L_{\text{JFDL}}$ only.} CIFAR-10 samples from Naive JFDL (\textbf{top left}) v.s. Random JFDL (\textbf{bottom left}). FID w.r.t. $\omega$ plot (\textbf{right}) reflects the stronger guidance effect from Random JFDL compared to Naive, causing the FID to diverge faster.}
\label{fig:prelim}
\end{figure}

\subsection{Preliminary Results}
\label{sec:prelim}
We show the effect of $L_\text{JFDL}$ in naive JFDL on CIFAR-10. We first require a CM that is trained to solve for the $\emptyset$–class ODE. We choose ECT as our baseline due to its fast convergence speed, which achieved one-step FID of 3.40. Starting from the ECT weights, we appended lightweight guidance embedding layers and fine-tuned the model exclusively on $L_\text{JFDL}$. Our preliminary results show this simple adaptation equips CFG style control post-hoc to the model (Fig. \ref{fig:prelim}). 

However, since most CMs do not explicitly learn the $\emptyset$–class ODE, Naive JFDL’s applicability is limited. Inspired by recent results that dispense with the need of a $\emptyset$-class for guidance \cite{sadat2024no}, we propose Random JFDL, which replaces the $\emptyset$ label with a random class drawn from the data distribution to construct the unconditional anchor. Surprisingly, this variant generates higher contrast images that aligns even closer with CFG. We provide the full Random JFDL in Appendix \ref{sec:appendix-algo}.

%% file: content/experiments.tex
\section{Experiments}
\label{sec:experiments}
We first discuss our experimental setup. Then, we report quantitative FID results demonstrating that JFDL outperforms the unguided ECT baseline when $\omega < 2$ in 1-step sampling. Finally, we visualize the guidance effect in the qualitative evaluation section. 

\paragraph{Setup.} Our preliminary results in Sec. \ref{sec:prelim} show that guidance yields FID improvements primarily when $\omega < 2$ in 1-step sampling, a trend consistent with CFG in DMs. Therefore, our main experiments fine-tune the model by sampling $\omega \in [1,2]$ and $\omega \in [1,4]$ on CIFAR-10 and ImageNet 64x64 respectively. We also observed that optimizing JFDL alongside the ECT objective leads to better FID scores. As shown in Alg. \ref{alg:naive-jfdl}, we optimize a multi-task loss combining $L_\text{ECT}$ and $L_\text{JFDL}$ with adaptive weighting via GradNorm \cite{chen2018gradnormgradientnormalizationadaptive}. For both experiments, we use the last $\mathrm{out\_conv}$ layer to compute the gradient norm of the two losses \cite{kim2023consistency}. Inspired by recent work on Truncated CMs \cite{lee2025truncatedconsistencymodels}, we further shift the time sampling distribution to a higher log-normal mean, with -0.5 for CIFAR-10 and -0.4 for ImageNet 64×64. Following ECT, we adopt the EDM architecture for CIFAR-10, and zero-initialize the guidance embedding weights following ControlNet \cite{zhang2023addingconditionalcontroltexttoimage}. In ImageNet 64x64, we adopt the EDM2-S architecture and set the magnitude preserving coefficient to 1e-3 for the attached guidance layers to stablize fine-tuning. Both experiments use a batch size of 64 and converge after training 1.92 million images (30k iterations).  Compared to the lightweight ECT baseline, JFDL’s fine-tuning consumes only about 7.5\% of ECT’s training data and around 30\% of its GPU hours. We present detailed ablation studies of our design choices in Appendix. \ref{sec:ablation}.

\begin{table}[t]
\centering
\caption{\textbf{Results on CIFAR-10 and ImageNet 64x64.} We train ECT first on a budget of 25.6M images, then tune it with Naive/Random JFDL on a budget of 1.92M images. The table shows the gains of FID posthoc compared to an unguided CM, with $\omega$ denoting the lowest recorded FID. }
\label{tab:guidance-results}
\resizebox{\linewidth}{!}{%
\begin{tabular}{lcccc}
\toprule
\textbf{Method} & \multicolumn{2}{c}{\textbf{CIFAR-10}} & \multicolumn{2}{c}{\textbf{ImageNet 64×64}} \\
\cmidrule(lr){2-3} \cmidrule(lr){4-5}
 & 1-step FID & 2-step FID & 1-step FID & 2-step FID \\
\midrule
ECT (baseline)               & 3.40  & 1.92  & 5.84  & 3.72      \\
ECT + Naive JFDL             & 3.29 ($\omega$=1.05)  &  2.18 ($\omega$=1.55)  & 4.38 ($\omega$=1.80)       & 3.17 ($\omega$=1.40)      \\ 
ECT + Random JFDL            & 3.24 ($\omega$=1.15) & 2.06 ($\omega$=1.05)     & 4.68 ($\omega$=1.80) & 3.32 ($\omega$=1.20) \\ 
\bottomrule
\end{tabular}%
}
\end{table}

% in the body
\begin{figure*}[t]
  \centering
  \begin{subfigure}[b]{0.25\linewidth}
    \includegraphics[width=\linewidth]{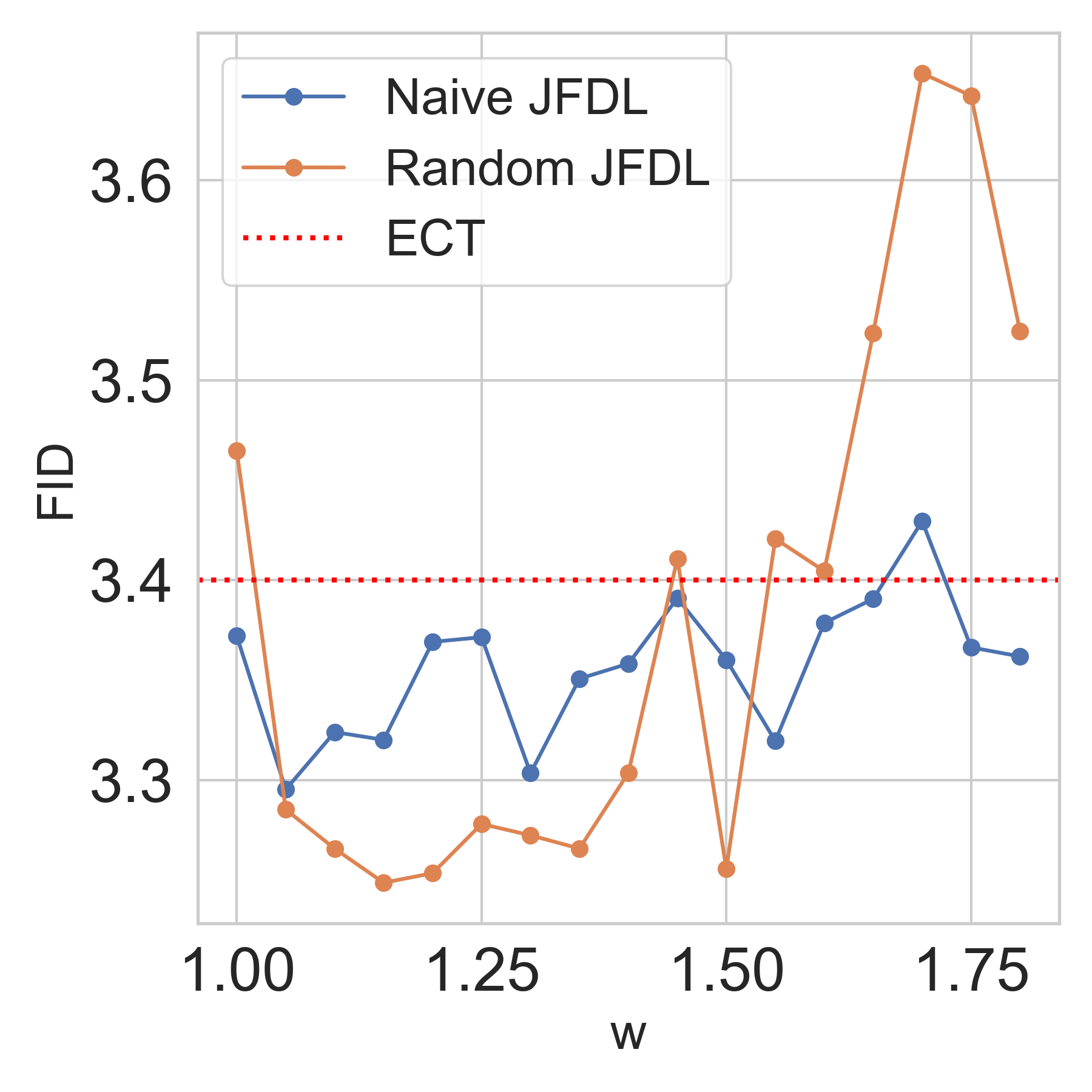}
    \caption{CIFAR-10 1-step}
  \end{subfigure}\hfill
  \begin{subfigure}[b]{0.25\linewidth}
    \includegraphics[width=\linewidth]{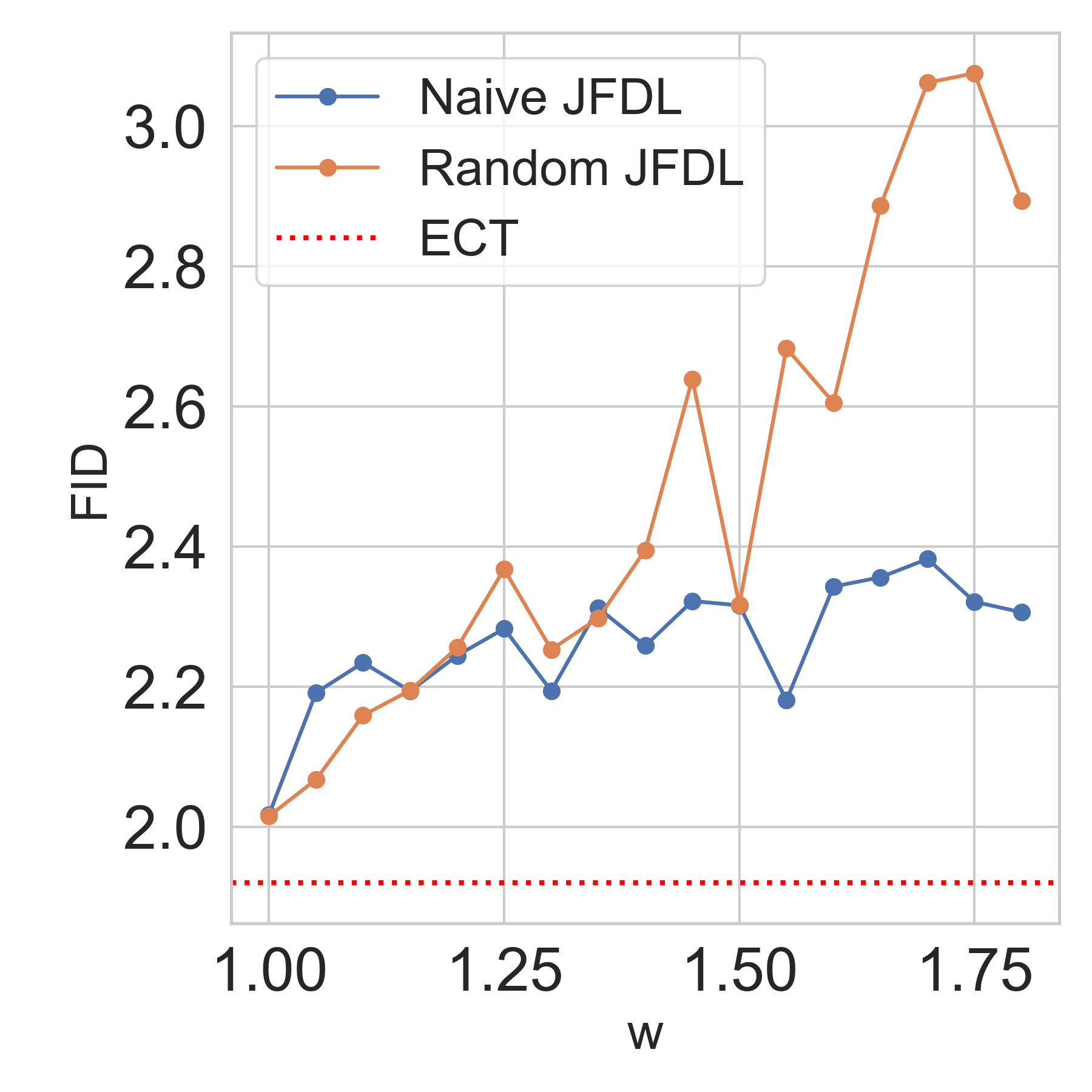}
    \caption{CIFAR-10 2-step}
  \end{subfigure}\hfill
  \begin{subfigure}[b]{0.25\linewidth}
    \includegraphics[width=\linewidth]{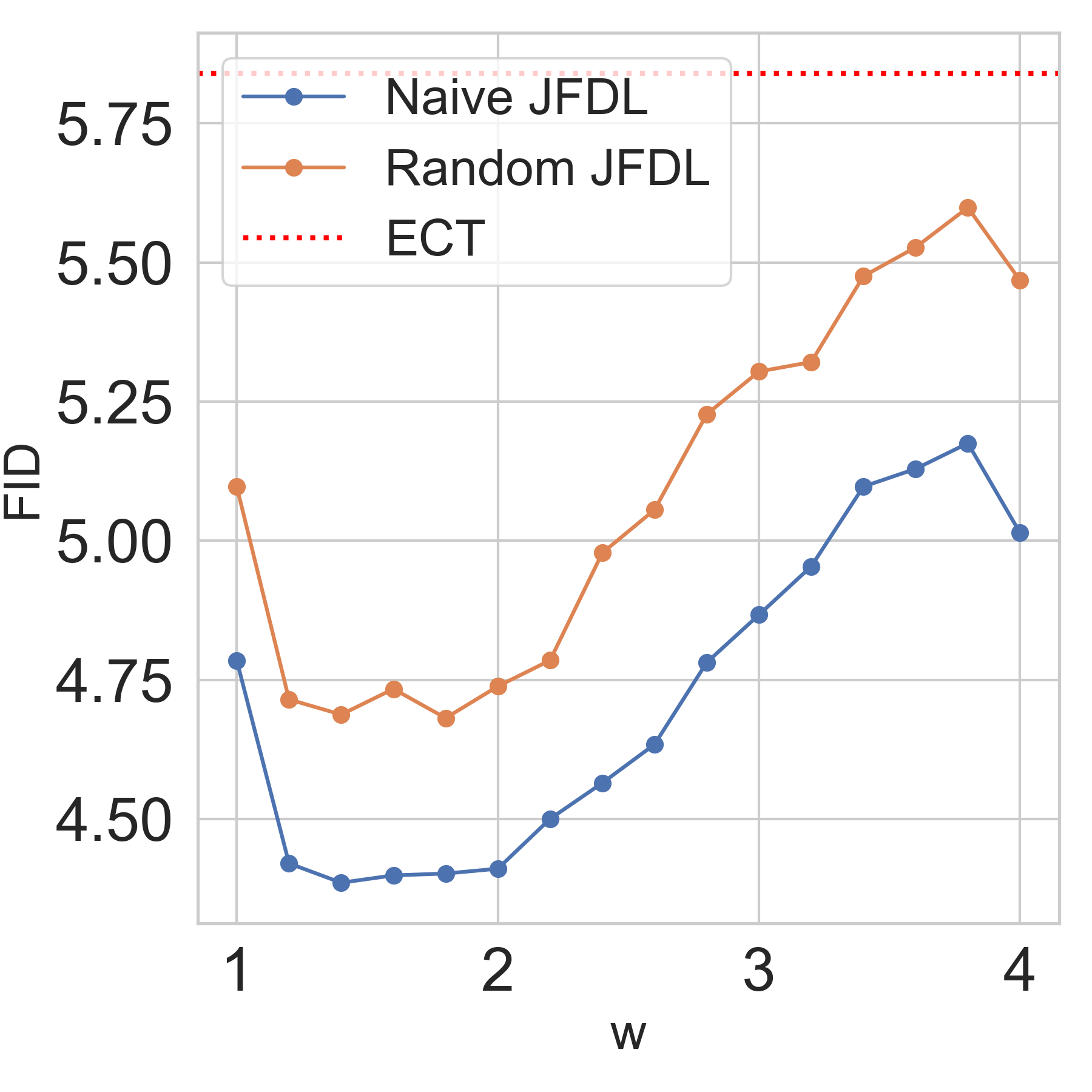}
    \caption{ImageNet 64x64 1‐step}
  \end{subfigure}\hfill
  \begin{subfigure}[b]{0.25\linewidth}
    \includegraphics[width=\linewidth]{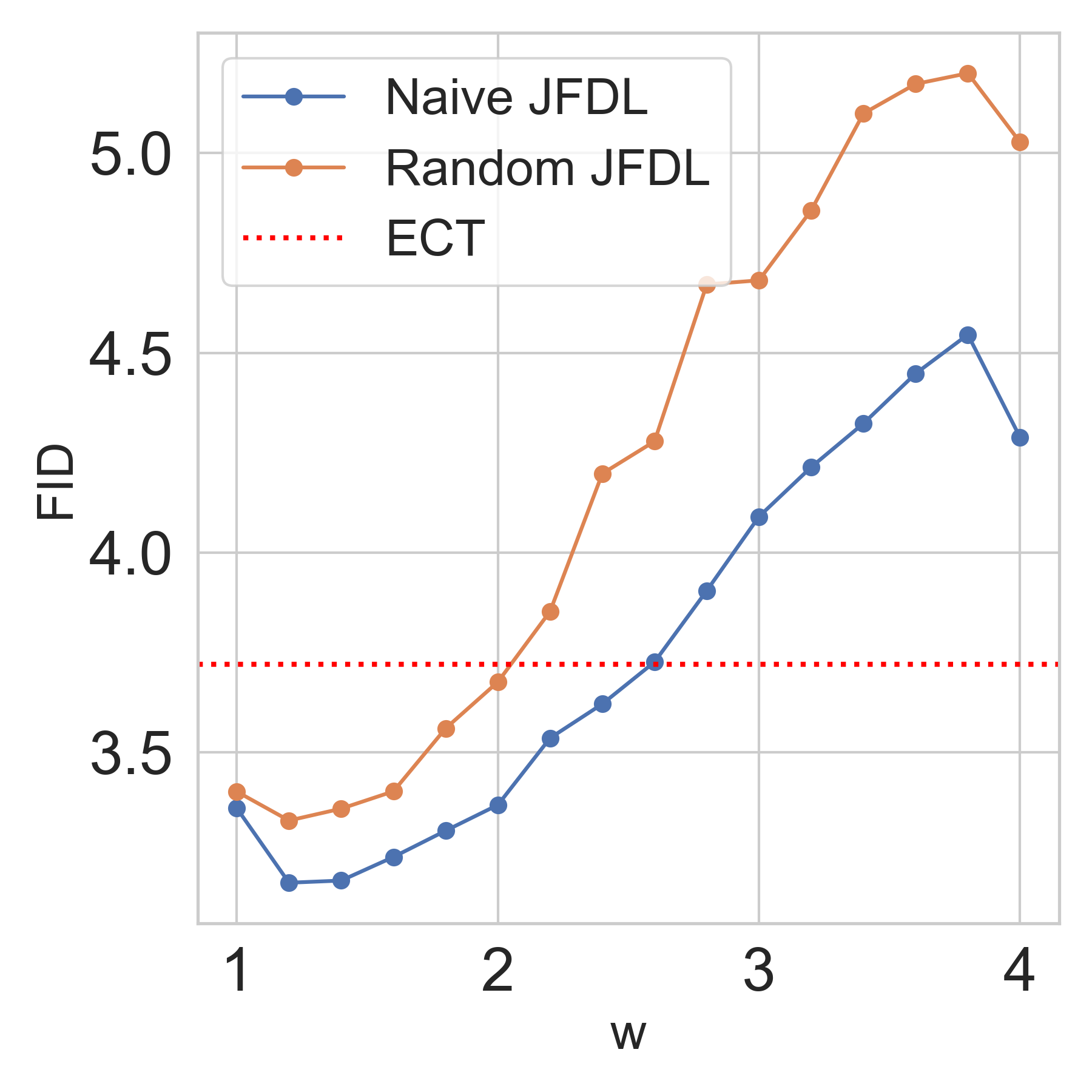}
    \caption{ImageNet 64x64 2‐step}
  \end{subfigure}

  \caption{\textbf{FID to guidance strength progression.} The red dotted line is the FID of the initial unguided ECT, the orange curve and blue curve represents the FID progression w.r.t $\omega$ for Naive and Random JFDL respectively. }
  \label{fig:combined_fid}
\end{figure*}

\paragraph{Quantitative Results.}
We evaluate generation quality using the standard FID-50k metric w.r.t. different $\omega$ scales \cite{heusel2018ganstrainedtimescaleupdate}. For each guidance weight $\omega$, we conducted three FID runs and report the mean in Tab. \ref{tab:guidance-results}. Our post-hoc tuning framework enables a direct comparison between guided models and the unguided ECT baseline. Compared to the preliminary results in Fig. \ref{fig:prelim} that only tuned on $L_\text{JFDL}$, jointly optimizing $L_\text{ECT}$ and $L_\text{JFDL}$ significantly reduces the FID compared to the baseline. Under 1-step sampling, both Naive and Random JFDL consistently lower FID on CIFAR-10 and ImageNet 64×64. The resulting FID curve over $\omega$ closely mirrors CFG's behavior in DMs, exhibiting improvements for small scales before degrading. For 2-step sampling, JFDL worsened FID on CIFAR-10, where the strong baseline meant the perceptual gains were outweighed by diversity loss and saturated style. However, JFDL improved FID on ImageNet 64×64, where there was more room for perceptual enhancement with few-step sampling under guidance. Notably, our experiments shows that the $\omega$=1 case deviates from the ECT baseline, even though in the case of CFG it is expected to be exactly the same as unguided generation. This can be explained by the fact that fine-tuning the original model (e.g., via JFDL) has altered the behavior of unguided sampling, whereas CFG is a training-free method for guiding DMs.

\begin{figure}[t]
\centering
\includegraphics[width=\linewidth]{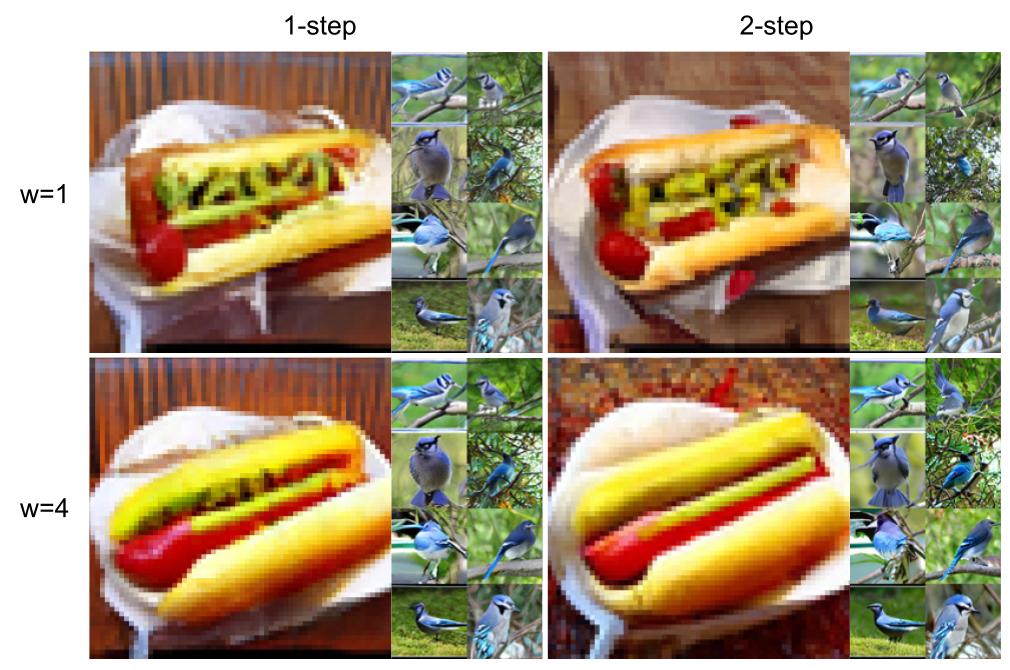}
\caption{\textbf{ImageNet 64x64 sample results.} The classes shown are "hotdog" and "jay", generated by ECT + Naive JFDL. Rows are guidance strength. Columns are sampling steps.}
\label{fig:ham-bird}
\end{figure}

\paragraph{Qualitative Results.}
Fig. \ref{fig:ham-bird} presents ImageNet 64×64 samples generated by JFDL under varying guidance scales and number of sampling steps. As the guidance scale $\omega$ increases (e.g. $\omega$=4), perceptual quality improves as the mustard appears crisper and the bun is more well-defined. We can also observe the diminishing sample diversity, as seen in the loss of certain hotdog toppings. This same fidelity-diversity trade‐off persists under 2‐step sampling, where the both steps uses the same $\omega$. To conclude, JFDL reproduces the effects of CFG, higher contrast and enhanced perceptual fidelity, demonstrating its promise as a post-hoc guidance method for CMs.

%% file: content/discussion.tex
\section{Discussion and Future Work}
\label{sec:conclusion}
We introduced JFDL, a fully post-hoc guidance framework for CMs that requires no DM teacher. By unifying perspectives from DMs, FMs and CMs, we derived theoretical guarantees for JFDL and validated them empirically. We further demonstrated that JFDL can equip an unguided ECT model with adjustable, CFG style guidance, yielding significant FID improvements on CIFAR-10 and ImageNet 64×64. 

In our work, we intentionally obscured the theoretical groundings for applying guidance under the multi-step sampling settings, but still applied it following previous works as in \cite{luo2023latentconsistencymodelssynthesizing}. In the context of CMs, one direction for future work is to explore the interaction between multi‐step sampling and guidance, and develop adaptive guidance schemes tailored to multi‐step solvers.

%% file: content/appendix.tex
% helpers
% top-level entry:  clickable text, flush-right page number
\newcommand{\tocentry}[2]{%
  \hyperref[#1]{#2}\hfill\pageref{#1}%
}

% second-level entry: clickable text, dotted leader, page number
\newcommand{\tocentrydots}[2]{%
  \hyperref[#1]{#2}\dotfill\pageref{#1}%
}

\section*{Supplementary Materials: Table of Contents}

\begingroup
\color{red}

% bullet-free, roomy list settings
\setlist[itemize,1]{
  label={},            % no bullet
  leftmargin=0.1cm,
  labelsep=0pt,
  itemsep=10pt,
  topsep=10pt
}
\setlist[itemize,2]{
  label={},
  leftmargin=1.0cm,
  labelsep=0pt,
  itemsep=10pt,
  topsep=10pt
}

\begin{itemize}
  \item \tocentry{sec:appendix-proofs}{A Proofs}
    \begin{itemize}
      \item \tocentrydots{sec:appendix-prop1}{A.1 Proof of Proposition 1}
      \item \tocentrydots{sec:appendix-prop2}{A.2 Proof of Proposition 2}
    \end{itemize}
  \item \tocentry{sec:appendix-algo}{B Random JFDL Algorithm}
  \item \tocentry{sec:appendix-exp}{C Experiments}
    \begin{itemize}
      \item \tocentrydots{sec:appendix-2d-toy}{C.1 2D Toy Dataset Setup}
      \item \tocentrydots{sec:cifar10-exp}{C.2 CIFAR-10 Experiment Setup}
      \item \tocentrydots{sec:imgnet-exp}{C.3 ImageNet 64×64 Experiment Setup}
%      \item \tocentrydots{sec:ablation}{C.4 Ablation Details}
    \end{itemize}
  \item \tocentry{sec:impact}{D Broader Impacts}
  \item \tocentry{sec:appendix-vis}{E Additional Visualizations}
\end{itemize}
\endgroup

% --------------------------------------------------------

\appendix
\setcounter{proposition}{0}

\section{Proofs}
\label{sec:appendix-proofs}
This section provides the theoretical support for Prop. \ref{prop:hybrid-distribution} and Prop. \ref{prop:pseudo-noise}. 

\subsection{Proof of Proposition \ref{prop:hybrid-distribution}}
\label{sec:appendix-prop1}

\begin{proposition}[Hybrid flow preserves the marginal data distribution]
Fix any noise index \(t\in[0,1]\).  
Draw a class label \(c\sim p(c)\) and a clean sample
\(\mathbf{x}_{0}^{c}\sim p_{0}(\mathbf{x}\mid c)\).
Define the hybrid ODE solution $\mathbf{y}_{0}^{\emptyset,t}:=\mathrm{Solver}_{t\to 0}^{\emptyset}\Bigl(\mathrm{Solver}_{0\to t}^{c}(\mathbf{x}_{0}^{c})\Bigr).$ We denote the density of $\mathbf{y}_{0}^{\emptyset,t}$ as $p(\mathbf{y}|c)$, which we define as:
\[
p(\mathbf{y}\mid c):=\bigl(\mathrm{Solver}_{t\to0}^{\emptyset}\bigr)_{\#}\bigl(\mathrm{Solver}_{0\to t}^{c}\bigr)_{\#}p_{0}(\mathbf{x}\mid c).
\]
Then the marginal law of \(\mathbf{y}_{0}^{\emptyset,t}\) over the class prior \(p(c)\) coincides with the unconditional-class data distribution:
\[
    \int_{\mathcal{C}} p(c)\,p(\mathbf{y}\mid c)\,\mathrm{d}c=p_0(\mathbf{x}).
\]
\end{proposition}

\begin{proof} \leavevmode\\[1ex]
For a function \(G:\mathbb{R}^d\mapsto\mathbb{R}^d\), and a probablity measure \(p\), the push-forward \(G_{\#}p\) is defined by:
\begin{equation}
\label{eq:push-forward}
    G_{\#}p(A) := p(G^{-1}(A)), \quad \forall \text{ measureable set}\ A.
\end{equation}
The set \(G^{-1}(A):=\{\mathbf{x}:G(\mathbf{x}) \in A\}\), is the pre-image of \(A\).

For any \(t\), from left hand side,
\begin{align*}
\int_{\mathcal C} p(c)\,p(\mathbf y \mid c)\,\mathrm{d}c
&=\int_{\mathcal C} p(c)\,
     \bigl(\mathrm{Solver}_{t\to 0}^{\emptyset}\bigr)_{\#}
     \bigl(\mathrm{Solver}_{0\to t}^{c}\bigr)_{\#}
     p_{0}(\mathbf{x}|c)\,\mathrm dc \\
&=\int_{\mathcal C} p(c)\,
     \bigl(\mathrm{Solver}_{t\to 0}^{\emptyset}\bigr)_{\#}p_t^c(\mathbf{x})\,\mathrm{d}c,
      \\
&\quad \quad \quad \quad \quad \quad \quad \quad \quad
     \textit{(define a measure }p_t^c(\mathbf{x}):=\bigl(\mathrm{Solver}_{0\to t}^{c}\bigr)_{\#}
     p_{0}(\mathbf{x}|c)\text{ )}
     \\
&=\int_{\mathcal C} p(c)\,p_t^c\bigl((\mathrm{Solver}_{t\to 0}^{\emptyset})^{-1}(\mathbf{x})\bigr)\  \mathrm{d}c, 
     \quad
     \textit{(by definition of \eqref{eq:push-forward})}
     \\
&=q\bigl((\mathrm{Solver}_{t\to 0}^{\emptyset})^{-1}(\mathbf{x}))\bigr),
     \quad 
     \textit{(define a measure}\ q(\mathbf{x}):=\int_{\mathcal C} p(c)\,p_t^c(\mathbf{x})\  \mathrm{d}c\text{)}  
     \\ 
&=\bigl(\mathrm{Solver}_{t\to 0}^{\emptyset}\bigr)_{\#}q(\mathbf{x}), 
    \quad 
    \textit{(by definition of \eqref{eq:push-forward})}
     \\
&=\bigl(\mathrm{Solver}_{t\to 0}^{\emptyset}\bigr)_{\#}\int_{\mathcal C} p(c)\,p_t^c(\mathbf{x})\  \mathrm{d}c, 
    \quad 
    \textit{(by definition of q)}
     \\
&=\bigl(\mathrm{Solver}_{t\to 0}^{\emptyset}\bigr)_{\#} p_t(\mathbf{x}), 
    \quad
    \textit{(by law of total probability)}
    \\
&=p_{0}(\mathbf x).
\end{align*}

\textit{Remark 1}: We assumed that each map $\mathrm{Solver}_{0\to t}^{c}$ and $\mathrm{Solver}_{t\to 0}^{\emptyset}$ is Borel-measurable so the push-forward definition \eqref{eq:push-forward} applies. The solvers are bijective with measurable (e.g.\ $C^{1}$) inverses, which ensures$(\mathrm{Solver}_{t\to0}^{\emptyset})^{-1}$ is well
defined.

\textit{Remark 2}:  The class prior satisfies $\int_{\mathcal C} p(c)\,p_{0}(\mathbf x\mid c)\,dc = p_0(\mathbf x)$, and analogously for the noisy measures $p_t(\mathbf x)=\int p(c)\,p_t^{c}(\mathbf x)\,dc$.

\end{proof}

\subsection{Proof of Proposition \ref{prop:pseudo-noise}}
\label{sec:appendix-prop2}

\begin{proposition}[Mixed-signals as a function of $t$]
Let $f(t) := \mathbf{y}_{0}^{\emptyset,t} = \mathrm{Solver}_{t\to 0}^{\emptyset}\!\bigl(\mathrm{Solver}_{0\to t}^{c}(\mathbf{x}_0^{c})\bigr)$ denote the hybrid flow, we can define the mixed-signals $g(t)$ as:
\[
    g(t) := f(0) + \frac{f(0) - f(t)}{t}.
\]    
then, from Taylor expansion to the 2nd order, we have:
\[
    g(t) \approx f(0)+\frac{t}{2}f''(0) = \mathbf{x}_0^c + \frac{\sigma_\text{max}t}{2}\nabla_{\mathbf x}\log p(c|\mathbf{x}_0^c). 
\]
\end{proposition}

\begin{proof} \leavevmode\\[1ex]
Since \(\mathbf{x}_0^c = f(0) = \mathrm{Solver}_{0\to0}^{\emptyset}\bigl( \mathrm{Solver}_{0\to 0}^c(\mathbf{x}_0^c)\bigr) \), as no push-forward results in the same output, the mixed signals \(g(t)\) can be expanded as,
\[
g(t) = \mathbf{y}_0^{\emptyset,t} + \frac{\mathbf{x}_0^c-\mathbf{y}_0^{\emptyset,t}}{t} = f(t) + \frac{f(0) - f(t)}{t}.
\]
Formally, we define the backward $\emptyset$-class PF-ODE \(\mathrm{Solver}_{t\to 0}^\emptyset(\mathbf{x})\) as the solution of: 
\begin{equation}
\label{eq:def-backward}
\frac{\mathrm{d} \mathbf{x}_s}{\mathrm{d}s} = v_s(\mathbf{x}_s) =- \sigma_s \nabla_{\mathbf x} \log p(\mathbf{x}_s), \quad \mathbf{x}_t = \mathbf{x},\quad s\in[0,t],
\end{equation}
and define the forward $c$-class PF-ODE \(\mathrm{Solver}_{0\to t}^c(\mathbf{x})\) as the solution of:
\begin{equation}
\label{eq:def-forward}
\frac{\mathrm{d} \mathbf{x}_s^c}{\mathrm{d}s} = u_s(\mathbf{x}_s^c) = - \sigma_s \nabla_{\mathbf x} \log p(\mathbf{x}_s^c|c), \quad \mathbf{x}_0^c = \mathbf{x},\quad s\in[0,t].
\end{equation}
The required differentiability and smoothness of the hybrid flow rely on the standard assumptions of continuity and Lipschitz continuity of the ODE vector fields guaranteed by the \textit{Picard–Lindelöf Theorem}. To proceed the derivation of the 2nd order Taylor expansion for \(g\), we first compute the 1st and 2nd derivatives of  \(f\), and then combine these results to construct the 2nd order Taylor expansion of \(f\) and obtain the final form. 

\textit{First derivative} \(f'(t)\):
$$
\begin{aligned}
f'(t)&=\frac{\mathrm{d} \ \mathrm{Solver}_{t\to0}^{\emptyset}\bigl( \mathrm{Solver}_{0\to t}^c(\mathbf{x}_0^c)\bigr)}{\mathrm{d}t} \\
&=\frac{\mathrm{d} \ \mathrm{Solver}_{t\to0}^{\emptyset}\bigl( \mathbf{x}_t^c\bigr)}{\mathrm{d}t} \\
&=D_{\mathbf x}\mathrm{Solver}_{t\to0}^{\emptyset}(\mathbf{x}_t^c) \cdot \underbrace{\frac{\mathrm{d}}{\mathrm{d}t}\mathbf{x}_t^c}_{(A)} + \underbrace{\frac{\partial}{\partial t}\mathrm{Solver}_{t\to0}^{\emptyset}(\mathbf{x}^c_t)}_{(B)},
\end{aligned}
$$
where we define \(\mathbf{x}_t^c:=\mathrm{Solver}_{0\to t}^c(\mathbf{x}_0^c)\) be the solution of the forward \(c\)-class ODE. The above result is obtained by the multivariate chain rule. We will continue to solve for terms \((A)\) and \((B)\). According to \eqref{eq:def-forward}, \((A)\) is, 
\begin{equation}
\label{eq:a}
    \frac{\mathrm{d}}{\mathrm{d}t}\mathbf{x}_t^c= - \sigma_t \nabla_{\mathbf x} \log p(\mathbf{x}_t^c|c).
\end{equation}
The term \((B)\), which can be interpreted as \textit{the change of time at \(t\) when the backward flow started, affecting the final solution at time 0}, we begin by definition of small time change \(t+\epsilon\), and derive the derivative by taking the limit. 
$$
\begin{aligned}
\mathrm{Solver}_{t+\epsilon\to0}^{\emptyset}(\mathbf{x}_t^c)
&=\mathrm{Solver}_{t\to0}^{\emptyset}
\bigl(
\mathrm{Solver}_{t+\epsilon\to t}^{\emptyset}(\mathbf{x}_t^c)
\bigr) \\
&=\mathrm{Solver}_{t\to0}^{\emptyset}
\bigl(
\mathbf{x}^c_t+\int^{t}_{t+\epsilon}- \sigma_\tau \nabla_{\mathbf x} \log p(\mathbf{x}_\tau^c) \ \mathrm{d}\tau
\bigr) \\
&=\mathrm{Solver}_{t\to0}^{\emptyset}
\bigl(
\mathbf{x}^c_t+\epsilon\cdot\sigma_t \nabla_{\mathbf x} \log p(\mathbf{x}_t^c)
\bigr) + o(\epsilon) \\
&= \mathrm{Solver}_{t\to0}^{\emptyset}
(\mathbf{x}^c_t)
+ \epsilon \cdot D_{\mathbf x}\mathrm{Solver}_{t\to0}^{\emptyset}(\mathbf{x}^c_t)\cdot\sigma_t \nabla_{\mathbf x} \log p(\mathbf{x}_t^c)
+ o(\epsilon). 
\end{aligned}
$$
Hence, by taking the limit \(\epsilon\to0\), we have, 
\begin{equation}
\label{eq:b}
\begin{aligned}
\frac{\partial}{\partial t}\mathrm{Solver}_{t\to0}^{\emptyset}(\mathbf{x}^c_t)
&= \lim_{\epsilon \to 0}\frac{\mathrm{Solver}_{t+\epsilon\to0}^{\emptyset}(\mathbf{x}_t^c)-\mathrm{Solver}_{t\to0}^{\emptyset}(\mathbf{x}_t^c)}{\epsilon} \\
&= D_{\mathbf x}\mathrm{Solver}_{t\to0}^{\emptyset}(\mathbf x)\cdot\sigma_t \nabla_{\mathbf x}\log p(\mathbf{x}_t^c).
\end{aligned}
\end{equation}
Combining \eqref{eq:a} and \eqref{eq:b} we get the first derivative:
\begin{equation}
\label{eq:first-d}
\begin{aligned}
f'(t) &= 
- \sigma_t \cdot D_{\mathbf x}\mathrm{Solver}_{t\to0}^{\emptyset}(\mathbf{x}_t^c)\cdot 
\big[ \nabla_{\mathbf x}\log p(\mathbf{x}_t^c|c) 
- \nabla_{\mathbf x}\log p(\mathbf{x}_t^c)
\big] \\
&\boxed{= 
- \sigma_t \cdot D_{\mathbf x}\mathrm{Solver}_{t\to0}^{\emptyset}(\mathbf{x}_t^c) \cdot 
\nabla_{\mathbf x}\log p(c|\mathbf{x}_t^c).}
\\
\end{aligned}
\end{equation}

To this point we must evaluate the Jacobian
\(D_{\mathbf{x}}\mathrm{Solver}_{t\to0}^{\emptyset}(\mathbf{x}_t^c)\),  
namely, how the backward solver’s \emph{output at time 0} varies under an infinitesimal perturbation of its \emph{input at time \(t\)}.

\begin{equation}\label{eq:b-jacobian}
\begin{aligned}
J_t &:=\frac{\partial}{\partial\mathbf{x}_t^c}\,
        \mathrm{Solver}_{t\to0}^{\emptyset}(\mathbf{x}_t^c)
       :=D_{\mathbf x}
        \mathrm{Solver}_{t\to0}^{\emptyset}(\mathbf{x}_t^c)
      \in\mathbb{R}^{d\times d}.
\end{aligned}
\end{equation}
We outline our derivation as follows. First, we state the definition of the backward trajectory, then linearize around that trajectory and describe a tiny perturbation with an ODE. Next, we introduce a \textit{state-transition} matrix \(J_{s,t}\). It is defined as the unique solution of a matrix ODE whose coefficients are exactly those from the linearized system. We show that the solution of our defined matrix ODE, \(J_{s,t}\), transforms any perturbation at time t to time s. Invoking the existence and uniqueness of solutions, we claim that the matrix transformation is the sole solution of the perturbation ODE equation, where the transformation itself can be defined by the matrix ODE. 

Recall that the backward $\emptyset$-class trajectory satisfies \(\mathrm{d}\mathbf{x}_s/\mathrm{d} s=v_s(\mathbf{x}_s), \mathbf{x}_t=\mathbf{x}_t^c\). Adding an infinitesimal perturbation \(\delta\mathbf{x}_s\) obeys,
\begin{equation}
\label{eq:linearized}
\begin{aligned}
    \frac{\mathrm{d}}{\mathrm{d}s}(\mathbf{x}_s+\delta\mathbf{x}_s)&=v_s(\mathbf{x}_s+\delta\mathbf{x}_s)\\
    &\approx v_s(\mathbf{x}_s)+D_\mathbf{x}v_s(\mathbf{x}_s)\delta\mathbf{x}_s,
\end{aligned}
\end{equation}
subtracting the unperturbed trajectory \(v_s(\mathbf{x_s})\), 
\begin{equation}
\frac{\mathrm{d}}{\mathrm{d}s}\delta\mathbf{x}_s=D_\mathbf{x}v_s(\mathbf{x}_s)\delta\mathbf{x}_s.
\end{equation}
Note that the above \(D_\mathbf{x}v_s(\mathbf{x}_s)\) is defined by the fixed trajectory in \eqref{eq:linearized}, which can be considered constant w.r.t. \(\delta\mathbf{x}_s\). For notation purpose, we abuse notation \(A(s):=D_{\mathbf x}v_s(\mathbf x_s)\). 

Next, define the \textit{state–transition matrix} \(J_{s,t}\in\mathbb R^{d\times d}\) by
\begin{equation}\label{eq:stm-def}
\frac{\partial}{\partial s}\,J_{s,t}=A(s)\,J_{s,t},
\qquad
J_{t,t}=I_d,
\end{equation}
where ~\eqref{eq:stm-def} guarantees a unique solution by the existence–uniqueness theorem. Let \(\delta\mathbf x_t\) be an arbitrary infinitesimal displacement at
time \(t\) and define
\(\tilde{\mathbf x}_s:=J_{s,t}\,\delta\mathbf x_t\).
Differentiating and using~\eqref{eq:stm-def},
\[
\frac{\mathrm d}{\mathrm ds}\tilde{\mathbf x}_s
     =A(s)\,J_{s,t}\,\delta\mathbf x_t
     =A(s)\,\tilde{\mathbf x}_s,
\quad
\tilde{\mathbf x}_t=I_d\,\delta\mathbf x_t=\delta\mathbf x_t.
\]
Hence \(\tilde{\mathbf x}_s\) satisfies the same linearised
ODE and initial condition as \(\delta\mathbf x_s\); by uniqueness,
\[
\delta\mathbf x_s = J_{s,t}\,\delta\mathbf x_t,
\]
which we successfully describe the perturbation from time t to time s with \(J_{s,t}\). Of special interest is the map from \(t\) to the final time \(0\): \(J_t:=J_{0,t}.\) Setting \(s=0\) in~\eqref{eq:stm-def} with \(t\) as the variable
parameter and differentiating the identity
\(J_{0,t}\,J_{t,0}=I_d\) gives
\[
\frac{\mathrm dJ_t}{\mathrm dt}
   =-A(t)\,J_t
   =-D_{\mathbf x}v_t(\mathbf x_t^c)\,J_t,
\qquad
J_0=I_d.
\]
If \(v_t(\mathbf x)=-\sigma_t\nabla_{\mathbf x}\log p(\mathbf x)\) then
\(A(t)=-\sigma_t D_{\mathbf x}^2\log p(\mathbf x_t^c)\), so
\begin{equation}\label{eq:J-ode-final}
\boxed{\;
\frac{\mathrm dJ_t}{\mathrm dt}
   =\sigma_t\,D_{\mathbf x}^2\log p(\mathbf x_t^c)\,J_t,
\qquad
J_0=I_d.}
\end{equation}
Equation~\eqref{eq:J-ode-final} is the \emph{variational (Jacobi) equation}: it uniquely propagates any perturbation vector at time \(t\) to its image at time \(0\).

\textit{Second derivative} \(f''(t)\):
\begin{equation}
\begin{aligned}
f''(t)=\frac{\mathrm{d}f'(t)}{\mathrm{d}t}
&=-\frac{\mathrm{d}}{\mathrm{d}t} \big[
\sigma_t \cdot J_t \cdot 
\nabla_{\mathbf{x}_t^c}\log p(c|\mathbf{x}_t^c) 
\big] \\
&=-\sigma_\text{max} \cdot 
J_t
\cdot 
\nabla_{\mathbf{x}_t^c}\log p(c|\mathbf{x}_t^c) \\
& \ \ \ \ \  
-\sigma_t \cdot 
\big[\ 
\underbrace{\frac{\mathrm{d}}{\mathrm{d}t}
J_t
}_{(C)}
\cdot 
\nabla_{\mathbf{x}_t^c}\log p(c|\mathbf{x}_t^c) + J_t \cdot \underbrace{\frac{\mathrm{d}}{\mathrm{d}t} 
\big(
\nabla_{\mathbf{x}_t^c}\log p(c|\mathbf{x}_t^c) 
\big)}_{(D)} \ 
\big],
\end{aligned}
\end{equation}
according to the first derivative \eqref{eq:first-d}. Term \((C)\) is the derivative of the \(J_t\) as described in \eqref{eq:J-ode-final}. Term \((D)\) is,
\begin{equation}
\label{eq:d}
\frac{\mathrm{d}}{\mathrm{d}t}\big(
\nabla_{\mathbf{x}_t^c}\log p(c|\mathbf{x}_t^c) 
\big)=D^2_{\mathbf{x}^c_t}\log p(c|\mathbf{x}_t^c) \frac{\mathrm{d}\mathbf{x}^c_t}{\mathrm{d}t}
=-D^2_{\mathbf{x}^c_t}\log p(c|\mathbf{x}_t^c)\cdot\sigma_t\nabla_{\mathbf{x}_t^c}\log p(\mathbf{x}_t^c|c),
\end{equation}
where \(\mathrm{d}\mathbf{x}_t^c /{\mathrm{d}t}\) is obtained by definition \eqref{eq:def-forward}. We can compose \(f''(t)\),
\begin{equation}
\label{eq:second-d}
\boxed{%
\begin{aligned}
f''(t)
= &-\,\sigma_{\text{max}} \ J_t \ \nabla_{\mathbf x}\log p(c|\mathbf x_t^c) \\
&-\sigma_t^2\cdot 
\Bigl[
      D_{\mathbf x}^2\log p(\mathbf x_t^c)\ J_t\
      \nabla_{\mathbf x}\log p(c|\mathbf x_t^c)
      -
      J_t\ D_{\mathbf x}^2\log p(c|\mathbf x_t^c) \ 
      \nabla_{\mathbf x}\log p(\mathbf x_t^c| c)
\Bigr] 
\end{aligned}%
}
\end{equation}

The 2nd order Taylor expansion of \(f\) and \(g\) is therefore,
\begin{equation}
\begin{aligned}
f(t) &= f(0) + tf'(0) + \frac{t^2}{2}f''(0) + o(t^2) \\
&= f(0) + t \cdot \mathbf{0} - \frac{\sigma_\text{max}t^2}{2}\nabla_\mathbf{x}\log p(c|\mathbf{x}_0^c) + o(t^2) \\
g(t) &= f(0)+\frac{f(0)-f(t)}{t}
\boxed{=f(0)+\frac{\sigma_\text{max}t}{2}\nabla_\mathbf{x}\log p(c|\mathbf{x}_0^c)+o(t^2)
}
\end{aligned}
\end{equation}
\end{proof}

% send paper to icg
\section{Random JFDL}
\label{sec:appendix-algo}

The authors in \cite{sadat2024no} introduced Independent Conditional Guidance (ICG) that enables guidance without special training of the \(\emptyset\)-class. Their approach achieved guided disribution even under unconditional sampling. Inspired by their work, we introduce Random JFDL. See Alg. \ref{alg:random-jfdl}, where the difference with Alg. \ref{alg:naive-jfdl} is highlighted in \textcolor{blue}{blue}. 
\begin{algorithm}[h]
  \caption{Random JFDL for Post-Hoc Guidance}
  \label{alg:random-jfdl}
  \begin{algorithmic}[1]
  
    \STATE \textbf{Input:} dataset $\mathcal{D}$, pre-trained CM $\psi$, weighting function $w(t)$, timesteps sampling density $p(t,r)$, total iterations $\mathrm{totalIters}$, max guidance scale $\omega_\text{max}$, gradnorm layer $\theta_\text{gn}$, gradnorm function $f(L_\text{ECT}, L_\text{JFDL} \ | \ \theta_\text{gn})$.
    
    \STATE \textbf{Init:} $\theta \leftarrow \psi$, $\mathrm{Iters} \leftarrow 0$.
    
    \WHILE{$\mathrm{Iters} < \mathrm{totalIters}$}
    
      \STATE Sample $(\mathbf{x}_0^c, c) \sim \mathcal{D}$, $t,r \sim p(t,r)$, $\mathbf{z} \sim \mathcal{N}(\mathbf{0},\mathbf{I})$, $\omega=1$;
      
      \STATE $\mathbf{x}_t^c \leftarrow \mathbf{x}_0^c + \sigma_t \mathbf{z}$; $\mathbf{x}_r^c \leftarrow \mathbf{x}_0^c + \sigma_r \mathbf{z}$; 
             
      \STATE $L_\text{ECT}(\theta) \leftarrow 
      w(t) \ d \bigl (G_{\theta}(\mathbf{x}_t^c, t, c, \omega), \ G_{\theta^-}(\mathbf{x}_r^c, r, c, \omega) \bigr)$; \hfill $\triangleright$ ECT loss

      \STATE \textcolor{blue}{Sample $c'\sim \mathcal{D};$}

      \STATE Sample $\mathbf{y}_0^c \leftarrow G_{\phi^-}(\mathbf{x}_t^c, t, c)$, \textcolor{blue}{$\mathbf{y}_0^{c',t} \leftarrow G_{\phi^-}(\mathbf{x}_t^c, t, c')$}, $\mathbf{z'} \sim \mathcal{N}(\mathbf{0},\mathbf{I})$, $\omega \sim \mathcal{U}(1,\omega_\text{max})$;

      \STATE $\mathbf{y}_t^c \leftarrow \mathbf{y}_0^c + \sigma_t \mathbf{z'}$; \textcolor{blue}{$\mathbf{y}_r \leftarrow \mathbf{y}_t^c + \{ \omega \left[ \frac{\mathbf{y}_t^c-\mathbf{y}_0^c}{\sigma_t} \right] + (1-\omega) \left[ \frac{\mathbf{y}_t^c-\mathbf{y}_0^{c',t}}{\sigma_t} \right] \} \cdot (\sigma_r-\sigma_t) $}; 

      \STATE $L_\text{JFDL}(\theta) \leftarrow 
      w(t) \ d \bigl (G_{\theta}(\mathbf{y}_t^c, t, c, \omega), \ G_{\theta^-}(\mathbf{y}_r, r, c, \omega) \bigr)$; \hfill $\triangleright$ JFDL loss

      \STATE $\lambda_\text{gn} = f(L_\text{ECT}, L_\text{JFDL} \ | \ \theta_\text{gn})$;
      \STATE $L(\theta) = L_\text{ECT}(\theta) + \lambda_\text{gn} L_\text{JFDL}(\theta)$;
              
      \STATE $\theta \gets \theta - \eta\nabla_\theta L(\theta)$;
      \STATE $\mathrm{Iters} \leftarrow \mathrm{Iters}+1$;
    \ENDWHILE
    \STATE \textbf{return} $\theta$ 
  \end{algorithmic}
\end{algorithm}

\section{Experimentals}
\label{sec:appendix-exp}
This section provides our experiment setup for the 2D toy experiments, CIFAR-10, and ImageNet 64x64. For all experiments, we adopt the diffusion VE scheme and select the $\sigma_\text{max}=80$ following previous works \cite{song2023consistency,kim2023consistency,karras2022elucidatingdesignspacediffusionbased,geng2024consistencymodelseasy}. We will disclose our code upon completion of the rebuttal cycle.

\subsection{2D Toy Dataset Setup}
\label{sec:appendix-2d-toy}
We train a DM to learn three 2D distributions, \textit{spiral}, \textit{circle}, and \textit{Gaussian blob} (see Fig. \ref{fig:toy-datasets-compare}). The DM is then used as an ODE to reconstruct the pseudo-noise for normality testing. 

\begin{figure}[htbp]
    \centering
    \begin{subfigure}[b]{0.32\textwidth}
        \centering
        \includegraphics[width=\textwidth]{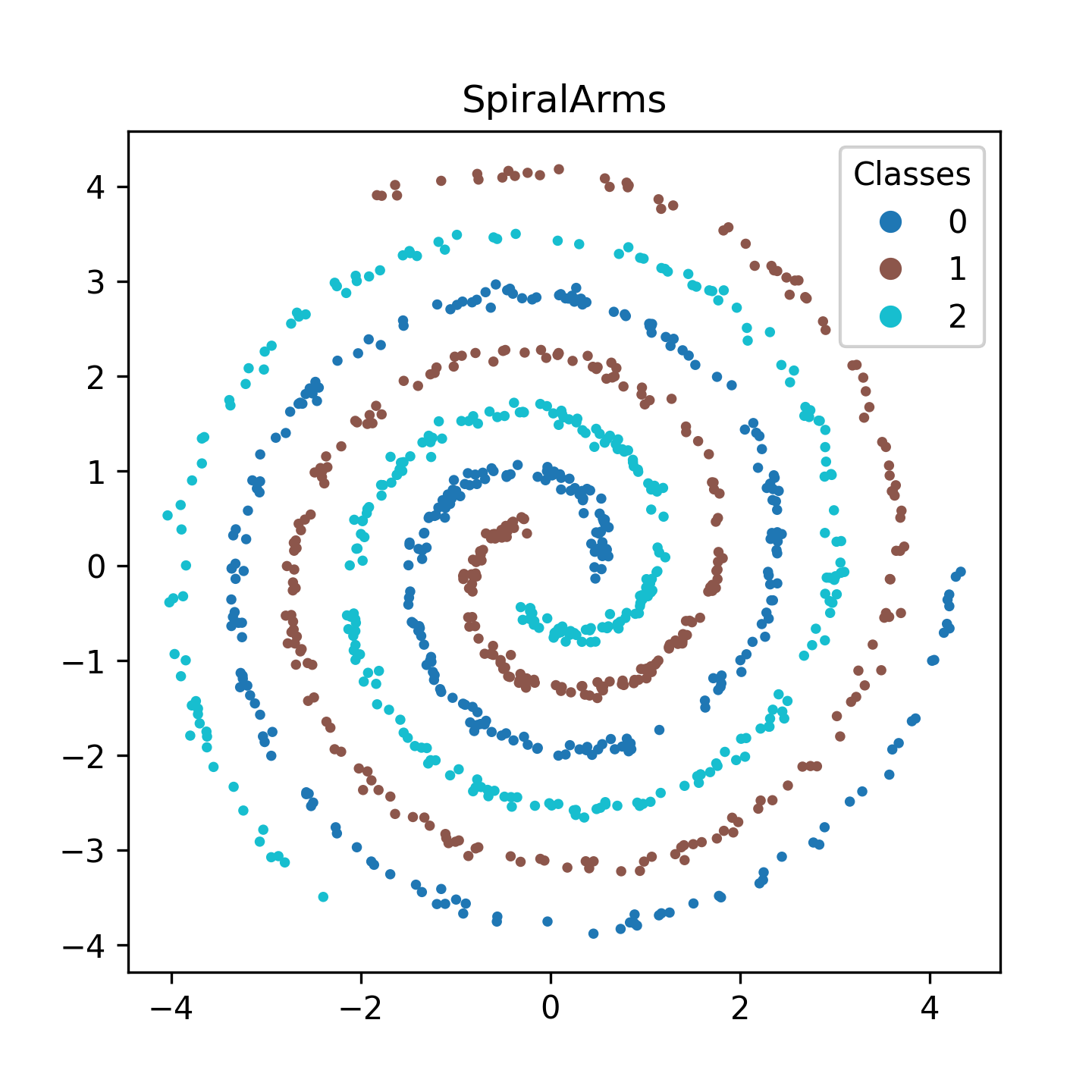}
        \caption{Spiral.}
    \end{subfigure}
    \begin{subfigure}[b]{0.32\textwidth}
        \centering
        \includegraphics[width=\textwidth]{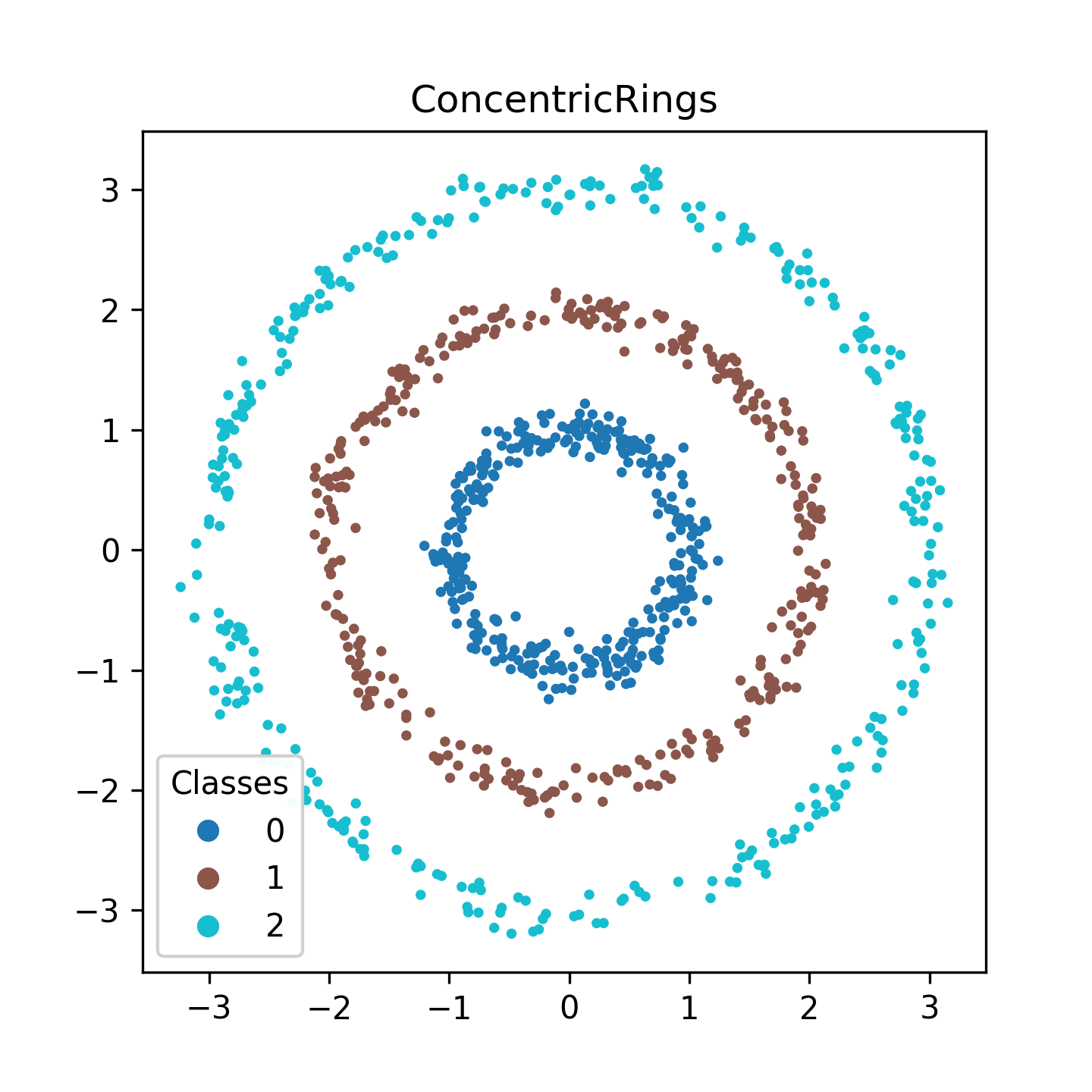}
        \caption{Circle.}
    \end{subfigure}
    \begin{subfigure}[b]{0.32\textwidth}
        \centering
        \includegraphics[width=\textwidth]{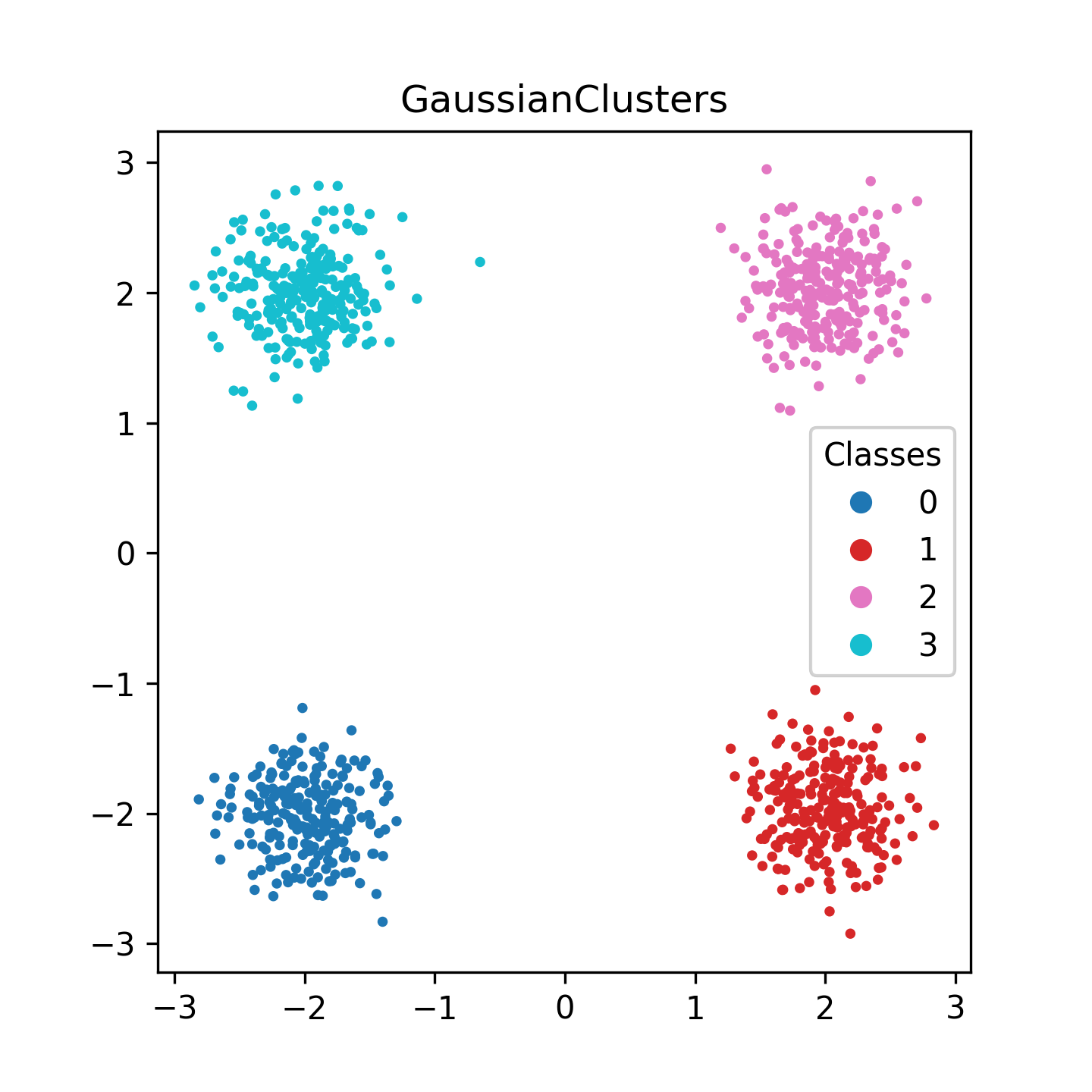}
        \caption{Gaussian blob.}
    \end{subfigure}
    \caption{Three synthetic 2D toy distributions used for normality verification experiments, spiral (left), circle (middle), and Gaussian blob (right). We train a DM that learns the class and $\emptyset$-class ODE. We use the DM as an ODE solver to reconstruct the pseudo-noise \eqref{eq:pseudo-noise} for any t, and test whether the pseudo-noise passes the normality test.}
    \label{fig:toy-datasets-compare}
\end{figure}

\paragraph{Dataset} The following datasets are constructed with specific geometric properties. We create a dataset of 10,000 samples from the described distribution for training.
\begin{itemize}
    \item \textbf{Spiral}: This dataset consists of three distinct spiral arms, each containing approximately 3,333 samples. The geometry of each arm is defined by a varying radius ($r = 0.5 + 0.3 \times \text{theta}$) as the angle ($\text{theta}$) sweeps from 0 to $4\pi$ radians. Each arm is then offset by $2\pi/3$ radians from the previous one, creating a visually distinct spiral structure. Gaussian noise with a standard deviation ($\sigma$) of 0.05 is added to the coordinates to introduce dispersion. The `sigma\_data` for this dataset is 1.85.
    \item \textbf{Circle}: This dataset is composed of three concentric circular rings, with approximately 3,333 samples per ring. The radii of the rings are 1.0, 2.0, and 3.0 units, respectively. Points for each ring are generated by sampling an angle ($\text{theta}$) uniformly from $0$ to $2\pi$ and then adding Gaussian noise with a standard deviation ($\sigma$) of 0.1 to the radius. This creates distinct, separated circular clusters. The `sigma\_data` for this dataset is 1.53.
    \item \textbf{Gaussian blob}: This dataset features four distinct Gaussian blobs (clusters), each containing approximately 2,500 samples. The clusters are centered at the following two-dimensional mean coordinates: $(-2, -2)$, $(2, -2)$, $(2, 2)$, and $(-2, 2)$. Each cluster's points are sampled from a Gaussian distribution with a standard deviation ($\sigma$) of 0.3, resulting in dense, well-separated spherical clusters. The `sigma\_data` for this dataset is 2.0.
\end{itemize}
  
\paragraph{Architecture} The DM network architecture adopts a Fourier embedding layer to handle time-dependent noise levels, and a \texttt{nn.Embedding} layer for class conditioning. The core of the model is a ResNet comprised of \texttt{BasicBlock} modules, which first projects the concatenated input features, i.e., noisy data, time embedding, and class embedding to a hidden dimension. Finally, a 1x1 convolution maps the hidden features back to the 2D data space. We follow EDM's preconditioning equations \cite{karras2022elucidatingdesignspacediffusionbased, song2021scorebasedgenerativemodelingstochastic} to achieve unit variance at training. 

\paragraph{Normality Testing} Following the procedures in Sec. \ref{sec:exp-verification} located under the experiment verification part, first we draw a class sample $\mathbf{x}^c_0$ and a noise $\mathbf{z}$. Then, for a chosen $t$, we can generate the distribution $p_t(\mathbf{x}_t|c)$, and push it back to t=0 with the ODE Solver to create class anchor $\mathbf{y}^c_0$, unconditional anchor $\mathbf{y}^{\emptyset,t}$ pair. We use the trained DM as the ODE Solver for our toy datasets, and the trained CM as for CIFAR-10. Following \eqref{eq:pseudo-noise}, we interpolate $\mathbf{y}^c_0$, $\mathbf{y}^{\emptyset,t}$ to create the mixed signal, then amplify a newly sampled $\mathbf{z}'$ by $\sigma_\text{max}$ to synthesize the pseudo-noise distribution. 

To assess the Gaussianity of the synthesized pseudo-noise, we employ a suite of statistical tests. Specifically, we conduct the Shapiro-Wilk \cite{SHAPIRO1965}, Anderson-Darling \cite{anderson1952asymptotic}, and Kolmogorov-Smirnov tests \cite{16e7f618-c06b-3d10-8705-1086b218d827}, each at a significance level of $\alpha=0.05$. The Shapiro-Wilk test evaluates the null hypothesis that the data is drawn from a normal distribution, with a high p-value (greater than $\alpha$) indicating a failure to reject normality. The Anderson-Darling test also assesses if the data comes from a specified distribution (in our case, a normal distribution after standardization), with a statistic smaller than the critical value at $\alpha$ indicating consistency with normality. Similarly, the Kolmogorov-Smirnov test compares the empirical cumulative distribution function of the sample data with that of a reference normal distribution, where a p-value above $\alpha$ suggests that the data follows the reference distribution. For our 2D toy examples, we conduct three statistical tests on 2000 random samples, equating to a size of 4000 data after flattening. On CIFAR-10, we sample 256 images, which equates to 786432 data points after flattening. Particularly, 5000 flatten data is randomly sampled for the Shapiro-Wilk test following the standard suggested maximum sample size. We report the pass/fail of each normality test in Fig.\ref{fig:normality}, where fail means we rejected normality due to low p-value. 

In addition to these hypothesis tests, we analyze the signal-to-noise ratio (SNR) between the mixed signal and the pseudo-noise. The SNR provides a quantitative measure of the noise power compared to the signal's power, calculated as, $$10\times\log_{10}\frac{\|\text{mixed-signals}\|^2}{\|\text{pseudo-noise}-\text{mixed-signals}\|^2}.$$ 
On the 2D toy distribution, we record the maximum SNR is \(\sim\!-5\), whereas on the CIFAR-10 dataset is \(\sim\!-27\), indicating very low signal residual in the pseudo-noise. These combined metrics help to understand the extent to which the added noise dominates the underlying signal at various $t$ values.

\subsection{CIFAR-10 Experimental Setup}
\label{sec:cifar10-exp}

We describe the architecture of guidance embedding layers appended to the baseline model at JFDL fine-tuning. Executing Naive JFDL in practice requires a conditional CM that solves the $\emptyset$-class, i.e., unconditional sampling. The idea is similar to training CFG enabled DM, where we substitute the class label with the $\emptyset$ class with probability mask \(q\). Unconditional distribution is essentially learned by,
\[
L_\text{ECT}(\theta)\leftarrow w(t) d(G_\theta(\mathbf{x}^c_t,t,\emptyset),G_{\theta^-}(\mathbf{x}^c_r,r,\emptyset)).
\]
We summarize the training hyper-parameters and GPU-hours in Tab. \ref{tab:hyperparameters}. 

\paragraph{Architecture} Our base model is a UNet configured under DDPM++ \cite{ho2020denoisingdiffusionprobabilisticmodels} following ECT. After the baseline ECT converges, we append it with our guidance embedding, which consists of a Fourier embedding layer that encodes the guidance scale \(\omega\), followed by three linear layers with SiLU activations \cite{elfwing2017sigmoidweightedlinearunitsneural}. The guidance embeddings are zero-initialized to stablize training in the beginning. 

\paragraph{Training} The training of the base model follows exactly ECT with a training budget of 256 million images. The only difference is that we additionally train the baseline to do unconditional sampling, with masking probability \(q\!=\!0.2\). Compared to ECT's result, the resulting FID for conditional sampling improved to 3.41/1.92 for 1- and 2-step sampling respectively. 

For both Naive and Random JFDL fine-tuning, our results in Tab. \ref{tab:guidance-results} were trained with a budget of 19.2 million images with a batch size of 64. To simulate the scenario where the baseline CM has not learned the $\emptyset$ distribution, we turn off $\emptyset$ masking during the ECT component \(L_\text{ECT}\) at fine-tuning. Compared to the preliminary results in Fid. \ref{fig:prelim}, retaining the ECT component improved FID by \(\sim\!1\). The guidance \(\omega\) is a continuous random number uniformly sampled from \(\mathcal{U}(1,2)\) for the results in Tab. \ref{tab:guidance-results}, and \(\mathcal{U}(1,5)\) for our preliminary results in Fig. \ref{fig:prelim}. Based on the findings from Truncated CMs \cite{lee2025truncatedconsistencymodels}, we sample the noise $\sigma_t$ following a log-normal distribution, $\textit{Lognormal}(P_\text{mean},P_\text{std})$, with higher mean, \(P_\text{mean}\!=\!-0.5\), to emphasize learning of higher and broader noise scales. 

\subsection{ImageNet 64x64 Experimental Setup}
\label{sec:imgnet-exp} 
The training of our baseline CM for ImageNet 64x64 overlaps the setup of ECT largely. Similar to the $\emptyset$-class learning described in Sec. \ref{sec:cifar10-exp}, we scale down \(q\!=\!0.1\) since there are 100x more classes in the ImageNet 64x64 dataset. A summary of the training hyper-parameters can be found in Tab. \ref{tab:hyperparameters}. The ImageNet 64x64 dataset experiments we reported in Tab. \ref{tab:guidance-results}, however, was configured following the pre-processing from the EDM codebase. We later realized there exists two major pre-processing procedures for the ImageNet 64x64 dataset, one following EDM, and the other following EDM2. This probably resulted in minor discrepancies for researchers trying to \cite{li2025bidirectionalconsistencymodels,song2023improved}. reproduce FID results from previous CM models, e.g. iCT, ECT, etc. As a result, our baseline FID reported in \ref{tab:guidance-results} were slightly higher than ECT's report, 5.84/3.72 for 1-/2-step, but still really close. 

\paragraph{Architecture} The base model is a UNet following EDM2 style, with sophisticated self-normalizing layers at training to preserve the magnitudes of each activation layer \cite{karras2024analyzingimprovingtrainingdynamics}. As a result, appending guidance layers naively would easily break the pretrained weights. % Add ablation before arxiv.
To retain the pretrained results while fine-tuning JFDL, we found out that joining the guidance layers with the class layers greatly stablised training. Additionally, zero-initialization is not practical following EDM2's style, as layers will self-normalize even if all weights are initialized as 0. To simulate the zero-initializing effect, we tuned down the weight to 1e-3 at the magnitude-preserving joining layers with the class-embeddings. Precisely, the aggregation concept is as follows, 
\[
\text{MP-Sum}(\mathbf{e}_\text{cls},\mathbf{e}_\omega, \alpha) = \frac{(1-\alpha)\cdot\mathbf{e}_\text{cls}+\alpha\cdot\mathbf{e}_\omega}{\sqrt{(1-\alpha)^2+\alpha^2}},
\]
where $\mathbf{e}_\text{cls}$ and $\mathbf{e}_\omega$ denote the class and guidance embedding vectors respectively, and $\alpha$=1e-3 denotes the weighted factor at magnitude-preserving summation. Our baseline model reported in Tab. \ref{tab:guidance-results} is the smallest standard EDM2, which is known as EDM2-S or ECM-S from previous works \cite{karras2024analyzingimprovingtrainingdynamics,geng2024consistencymodelseasy}. 

\paragraph{Training} The results in Tab. \ref{tab:guidance-results} were trained with a budget of 19.2 million images with a batch size of 64. The guidance $\omega$ is uniformly sampled from \(\mathcal{U}(1,4)\). In contrast to previous works \cite{karras2022elucidatingdesignspacediffusionbased,geng2024consistencymodelseasy,karras2024analyzingimprovingtrainingdynamics}, we sample the noise $\sigma_t$ following a log-normal distribution with higher mean, $P_\text{mean}\!=\!-0.4$.

\begin{table}[t]
  \caption{Hyperparameters used for CIFAR-10 and ImageNet 64x64 experiments. Baseline represents the CM that learns the $\emptyset$ distribution, which majorly follows ECT. Both Naive and Random JFDL in Tab. \ref{tab:guidance-results} share the same training hyperparameters, where we summarize under the JFDL column. The CIFAR-10 experiments has an additional preliminary column that concludes the settings in Fig. \ref{fig:prelim}. }
  \label{tab:hyperparameters}
  \centering
  % --- Start of adjustbox ---
  \begin{adjustbox}{width=0.8\textwidth,center}
  \begin{tabular}{lccccc}
    \toprule
    Hyperparameters & \multicolumn{3}{c}{CIFAR-10} & \multicolumn{2}{c}{ImageNet 64x64} \\
    \cmidrule(lr){2-4} \cmidrule(lr){5-6}
     & Baseline & Prelim & JFDL & Baseline & JFDL \\
    \midrule
    Architecture & DDPM++ & DDPM++ & DDPM++ & ECM-S & ECM-S \\
    Learning rate & 0.0001 & 0.0001 & 0.0001 & 0.0010 & 0.0007 \\
    Learning rate decay & - & - & - & 2K & 2K \\   
    Optimizer & RAdam & RAdam & RAdam & Adam & Adam  \\
    EMA decay rate & 0.9993 & 0.9993 & 0.9993 & - & - \\
    Posthoc EMA rate & - & - & - & 0.050 & 0.100 \\
    Training iterations & 2M & 0.4M & 0.3M & 2M & 0.3M \\
    Probability q of $\emptyset$ mask & 0.2 & - & - & 0.1 & - \\
    Batch size & 128 & 128 & 64 & 128 & 64 \\
    $\omega_\text{max}$ & - & 5 & 2 & - & 4 \\
    $P_\text{mean}$ & -1.1 & -1.1 & -0.5 & -0.8 & -0.4 \\
    $P_\text{std}$ & 2.0 & 2.0 & 2.0 & 2.0 & 2.0 \\
    GPU type & a5000 & a5000 & l40s & l40s & l40s \\
    Number of GPUs, i.e., \#GPUs & 4 & 4 & 1 & 4 & 2 \\
    GPU hours, after \(\times\)\#GPUs & 69 & 47.5 & 2.5 & 87 & 26.5 \\
    Acceleration tf32 & True & True & True & True & True \\
    Dropout & 0.2 & 0.2 & 0.2 & 0.4 & 0.4 \\
    Dropout Resolution & - & - & - & 16 & 16 \\
    \bottomrule
  \end{tabular}
  \end{adjustbox}
  % --- End of adjustbox ---
\end{table}

%\subsection{Ablation details}
%\label{sec:ablation}

%\paragraph{Sampling distribution.} \textbf{TODO} c10 m=-1.1 (done) vs m=-0.5 (done); imgnet (ecm-s) m=-0.8 (done) vs m=-0.4 (done).

%\paragraph{Model size.} \textbf{TODO} imgnet (ecm-s, if have time to rerun with geng's checkpoint) (ecm-m, running now, done by 5/20) (ecm-l, queuing now on gpu-he 4 h100s). 

%\paragraph{Few step guided sampling.} \textbf{TODO} algorithm for sampling, comparison for two step -> w then w, w then 1, 1 then w, baseline 1 then 1 (unguided).  

\section{Broader Impacts}
\label{sec:impact}
Our work on post-hoc guidance for CMs advances the state of the art in fast, high-quality image generation, enabling applications in scientific visualization, accessibility, and creative content production. By reducing inference cost and tuning overhead, JFDL makes powerful generative tools more broadly available, including in resource-constrained settings. However, improving the control over image fidelity also increases risks around deepfake creation, automated misinformation, and biased outputs inherited from training data. We encourage practitioners to pair CMs with robust content authentication methods and to audit model outputs for fairness and privacy before deployment.

\section{Additional Visualizations}
\label{sec:appendix-vis}

\begin{figure}[t]
\centering
\includegraphics[width=\linewidth]{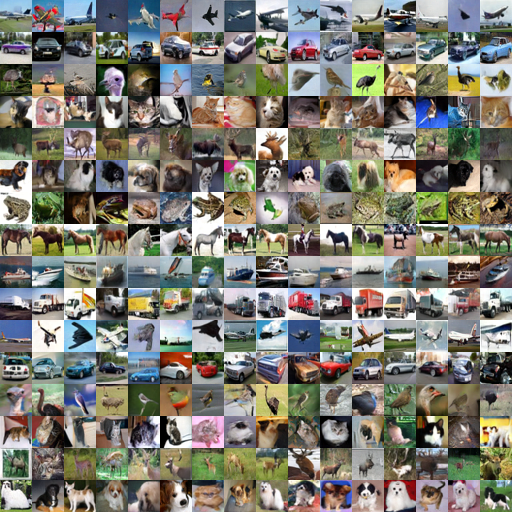}
\caption{CIFAR-10, Naive JFDL, 1-step, $\omega$=1.25}
\end{figure}

\begin{figure}[t]
\centering
\includegraphics[width=\linewidth]{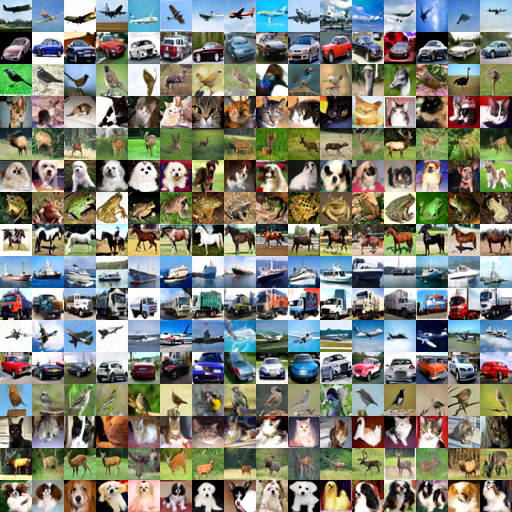}
\caption{CIFAR-10, Naive JFDL, 1-step, $\omega$=4.0}
\end{figure}

\begin{figure}[t]
\centering
\includegraphics[width=\linewidth]{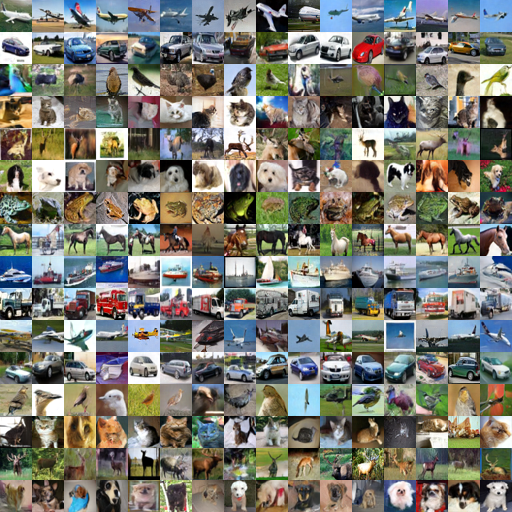}
\caption{CIFAR-10, Random JFDL, 1-step, $\omega$=1.25}
\end{figure}

\begin{figure}[t]
\centering
\includegraphics[width=\linewidth]{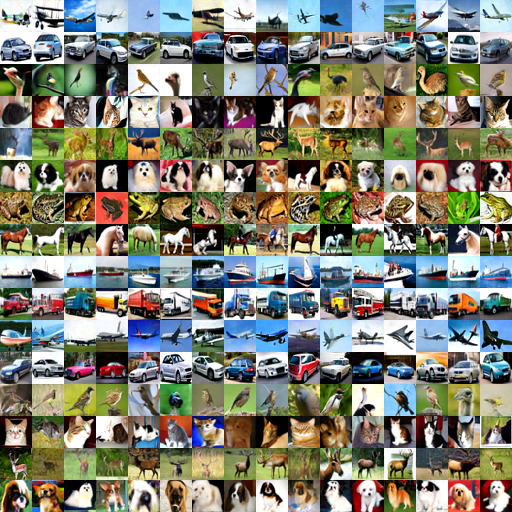}
\caption{CIFAR-10, Random JFDL, 1-step, $\omega$=4.0}
\end{figure}

\begin{figure}[t]
\centering
\includegraphics[width=\linewidth]{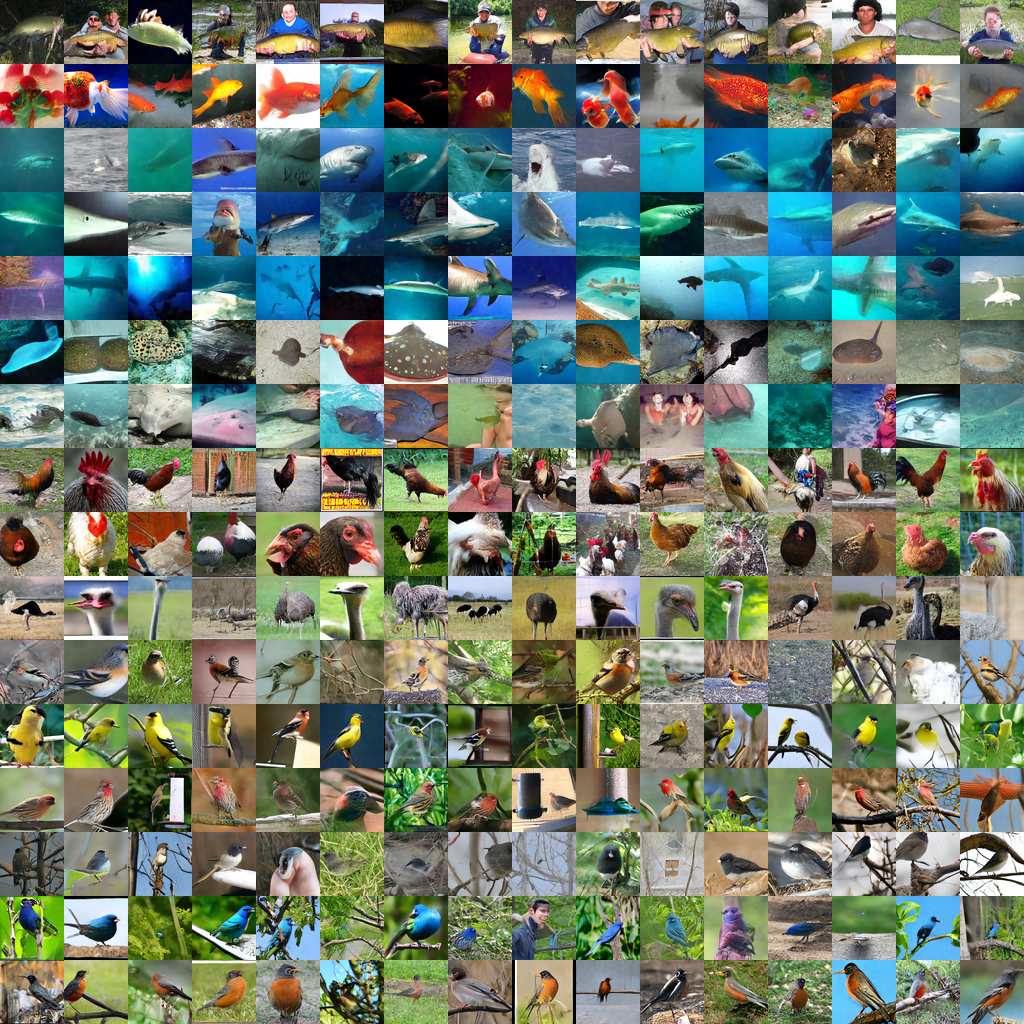}
\caption{ImageNet 64x64, Naive JFDL, 1-step, $\omega$=1.0}
\end{figure}

\begin{figure}[t]
\centering
\includegraphics[width=\linewidth]{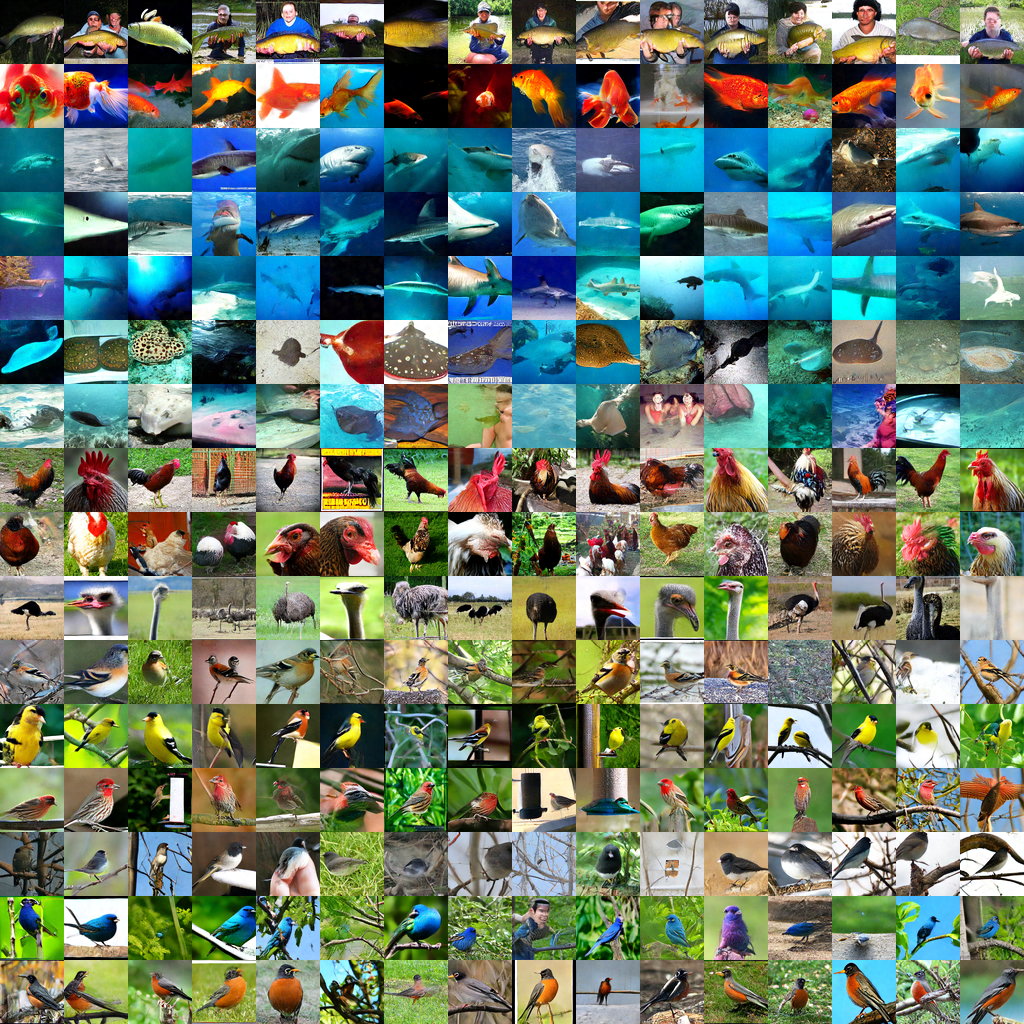}
\caption{ImageNet 64x64, Naive JFDL, 1-step, $\omega$=3.0}
\end{figure}

\begin{figure}[t]
\centering
\includegraphics[width=\linewidth]{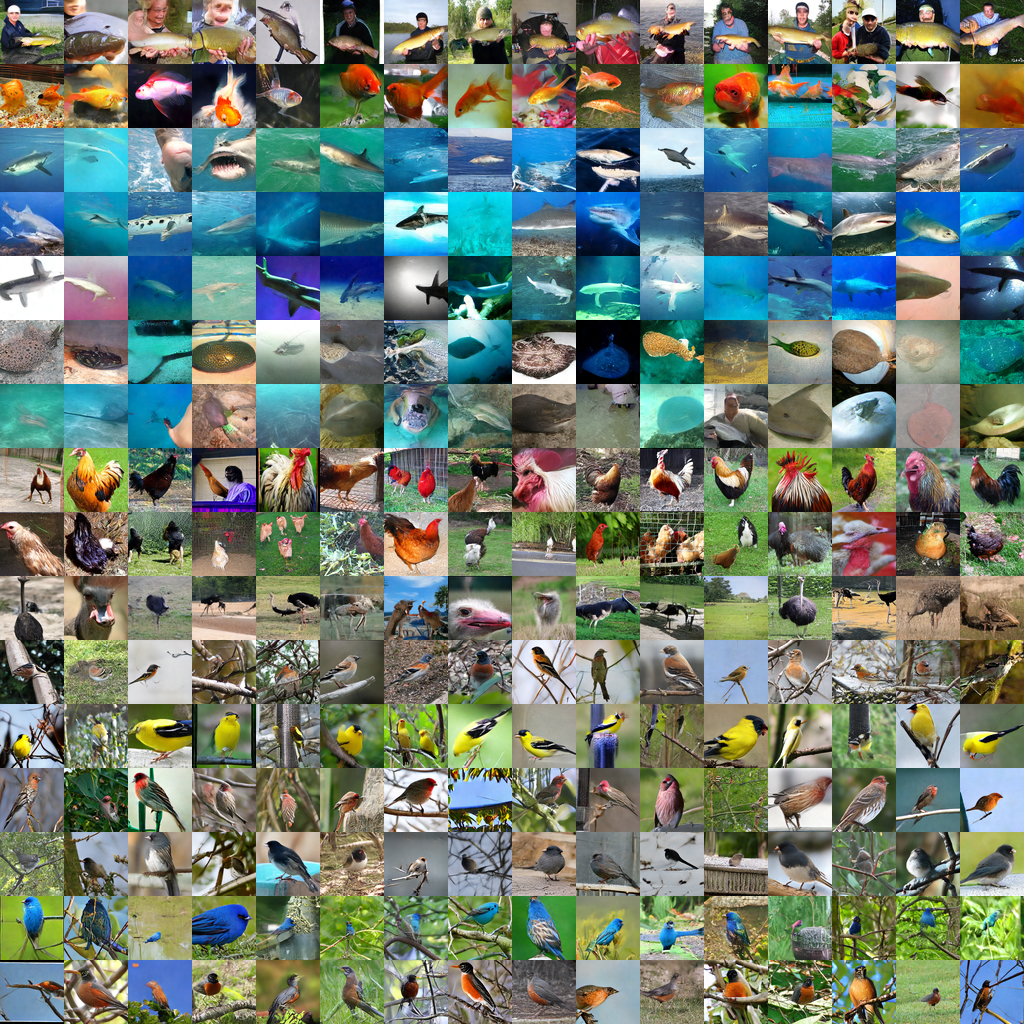}
\caption{ImageNet 64x64, Random JFDL, 1-step, $\omega$=1.0}
\end{figure}

\begin{figure}[t]
\centering
\includegraphics[width=\linewidth]{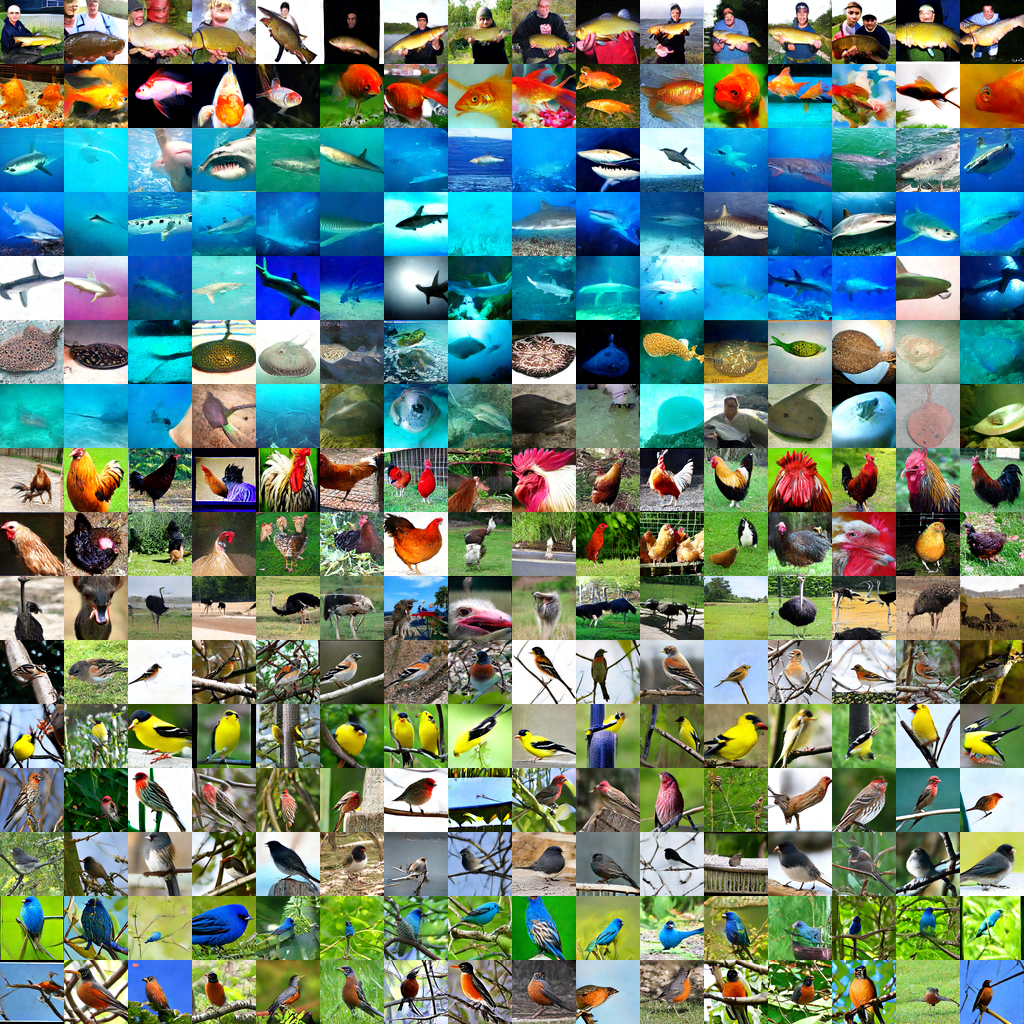}
\caption{ImageNet 64x64, Random JFDL, 1-step, $\omega$=3.0}
\end{figure}

\begin{figure}[t]
\centering
\includegraphics[width=\linewidth]{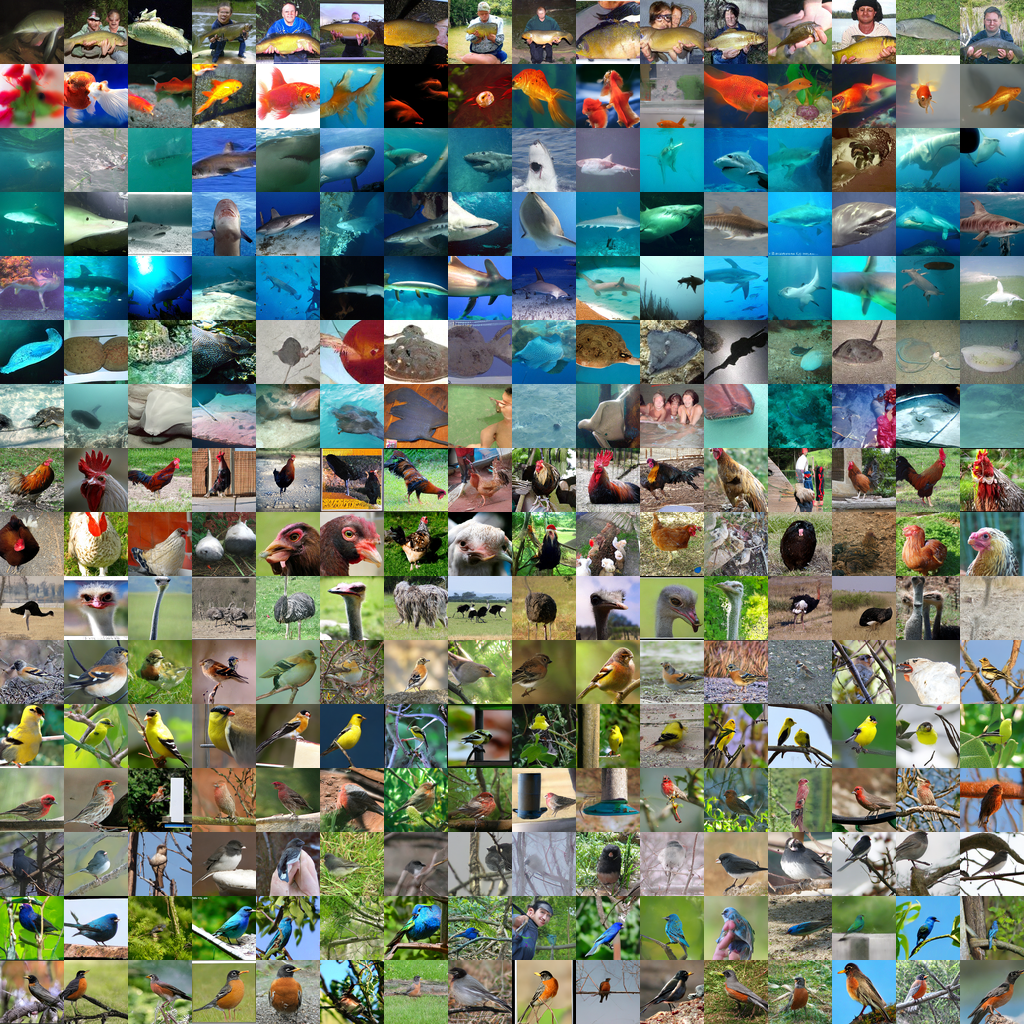}
\caption{ImageNet 64x64, Naive JFDL, 2-step, $\omega$=1.0}
\end{figure}

\begin{figure}[t]
\centering
\includegraphics[width=\linewidth]{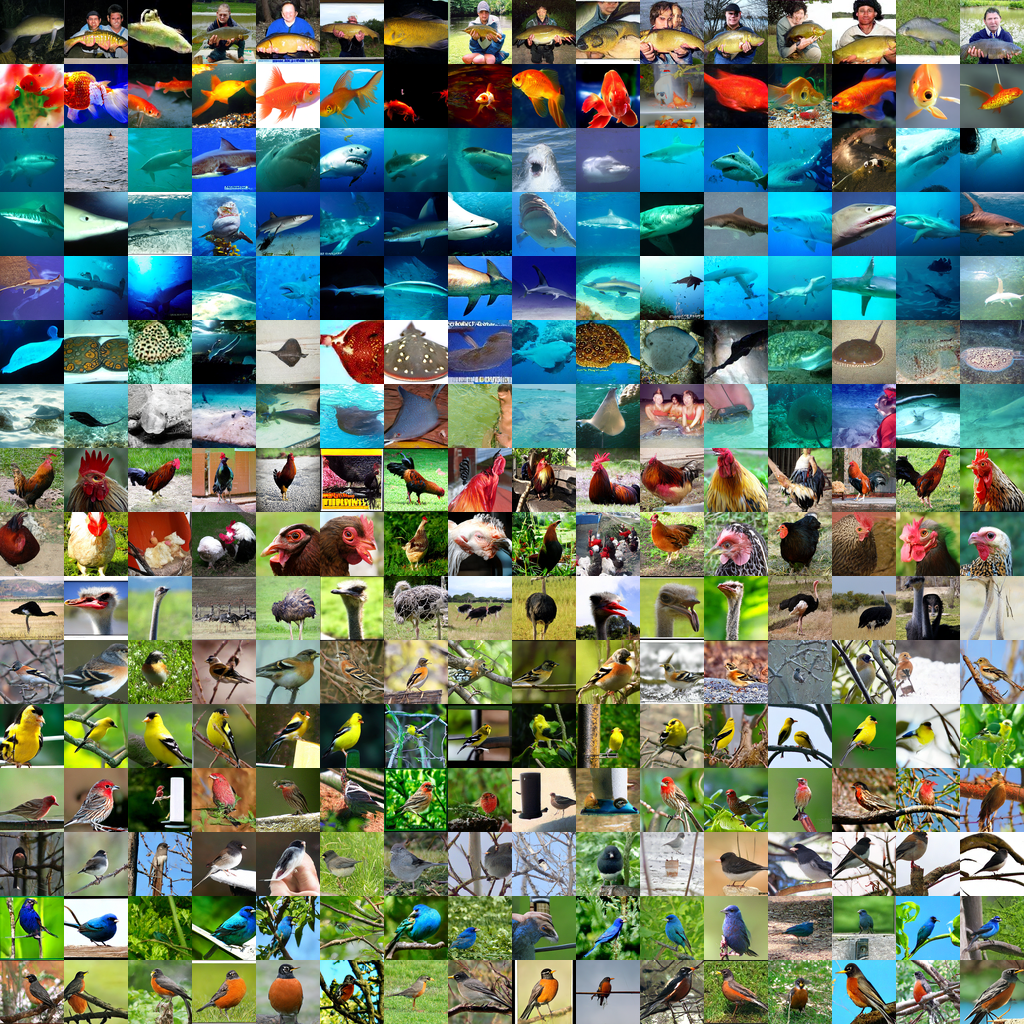}
\caption{ImageNet 64x64, Naive JFDL, 2-step, $\omega$=3.0}
\end{figure}

\begin{figure}[t]
\centering
\includegraphics[width=\linewidth]{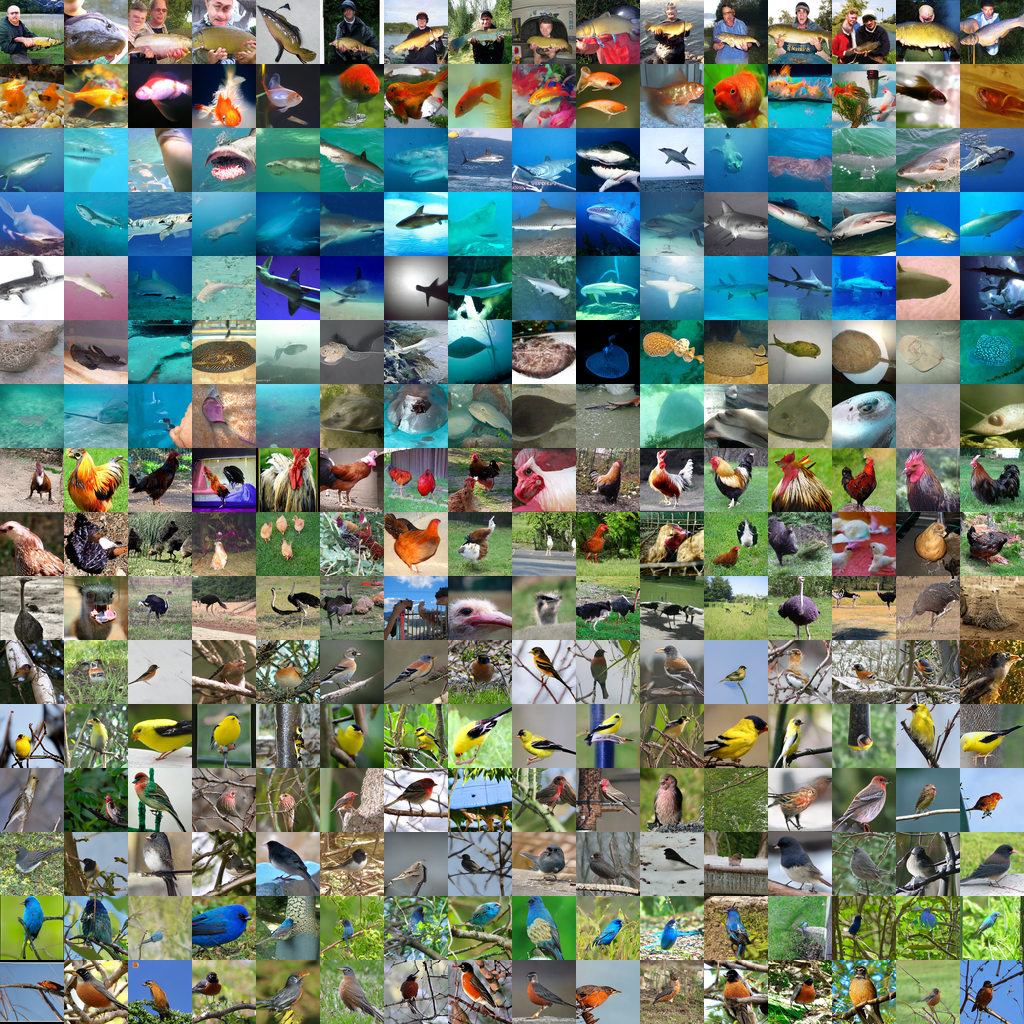}
\caption{ImageNet 64x64, Random JFDL, 2-step, $\omega$=1.0}
\end{figure}

\begin{figure}[t]
\centering
\includegraphics[width=\linewidth]{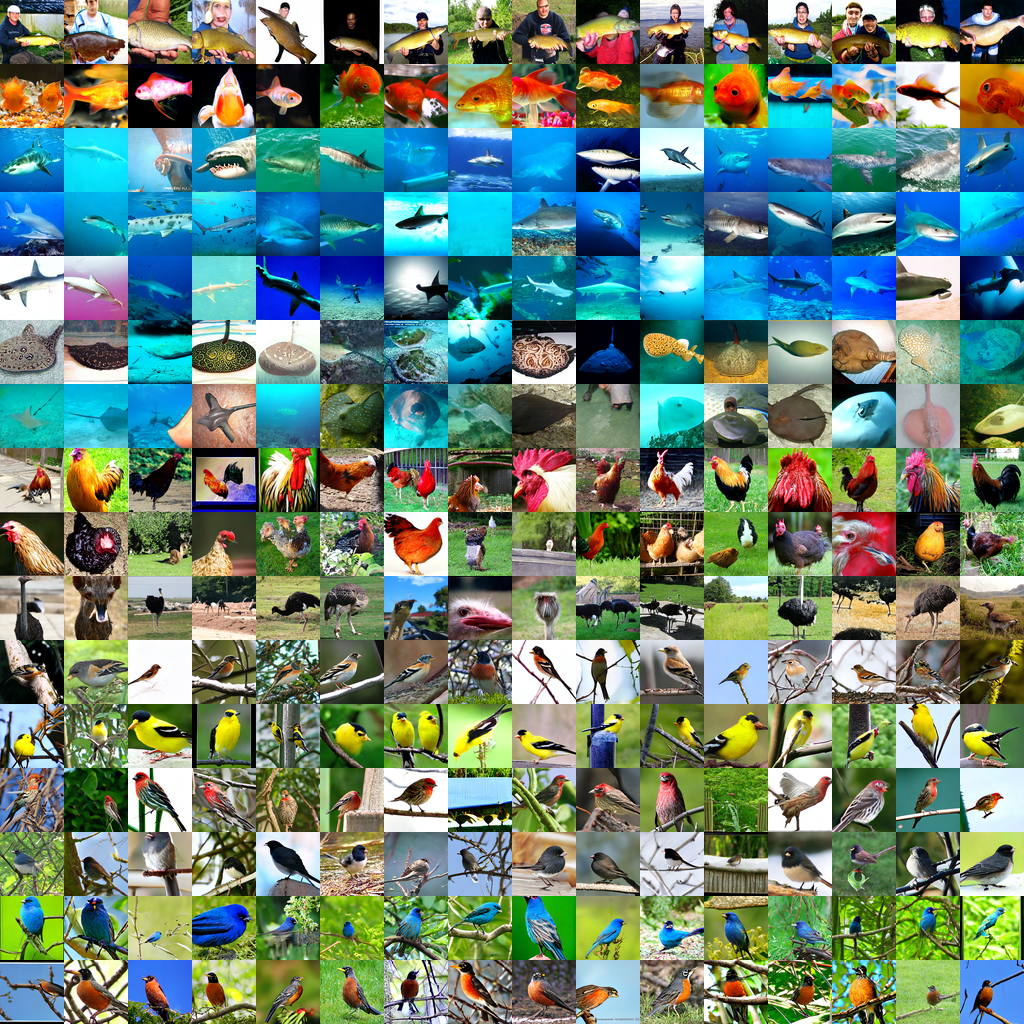}
\caption{ImageNet 64x64, Random JFDL, 2-step, $\omega$=3.0}
\end{figure}